\RequirePackage{fix-cm}
\PassOptionsToPackage{table}{xcolor}
\documentclass{article}

\usepackage[T1]{fontenc}
\usepackage[utf8]{inputenc}

\usepackage{booktabs}
\usepackage{multirow}
\usepackage[table]{xcolor}
\usepackage[most]{tcolorbox}
\usepackage{caption}
\usepackage{xspace}
\usepackage{enumitem}
\usepackage{graphicx}
\usepackage{subcaption}
\usepackage{colortbl}
\usepackage{amsmath}
\usepackage{amssymb}
\usepackage{amsfonts}
\usepackage{amsthm}
\usepackage{nicefrac}
\usepackage{mathtools}
\usepackage{siunitx}
\usepackage{makecell,threeparttable}
\usepackage{pifont}
\usepackage{wrapfig}
\usepackage{float}
\usepackage{rotating}
\usepackage{tabularx}
\usepackage{longtable}
\usepackage{etoc}
\usepackage{listings}
\usepackage{soul}
\usepackage{algorithm}
\usepackage{algorithmic}

\usepackage{tikz}
\usetikzlibrary{shadows,arrows.meta,positioning,shapes.geometric,calc,fit,backgrounds,decorations.pathreplacing}
\usepackage[edges]{forest}
\usepackage{pgfplots}
\pgfplotsset{compat=1.18}

\usepackage{arxiv_version}
\usepackage{xurl}
\hypersetup{colorlinks=true,citecolor=darkblue,linkcolor=darkblue,urlcolor=darkblue}
\usepackage[capitalise,noabbrev]{cleveref}

\crefname{section}{Section}{Sections}
\Crefname{section}{Section}{Sections}
\crefrangelabelformat{section}{#3#1--#4#2}
\crefname{subsection}{\S}{\S}
\Crefname{subsection}{\S}{\S}

\definecolor{lightcoral}{rgb}{0.94,0.5,0.5}
\definecolor{darkpastelgreen}{rgb}{0.01,0.75,0.24}
\definecolor{hidden-red}{RGB}{205,44,36}
\definecolor{hidden-blue}{RGB}{194,232,247}
\definecolor{hidden-orange}{RGB}{243,202,120}
\definecolor{hidden-green}{RGB}{34,139,34}
\definecolor{hidden-pink}{RGB}{255,245,247}
\definecolor{hidden-black}{RGB}{20,68,106}
\definecolor{purple}{RGB}{144,153,196}
\definecolor{yellow}{RGB}{255,228,123}
\definecolor{hidden-yellow}{RGB}{255,248,203}
\definecolor{tkcolor}{RGB}{224,223,255}
\definecolor{myred}{RGB}{247,226,231}
\definecolor{myblue}{RGB}{216,226,234}
\definecolor{myyellow}{RGB}{252,238,221}
\definecolor{mypurple}{RGB}{233,229,241}
\definecolor{mygreen}{RGB}{204,231,207}

\theoremstyle{plain}

\theoremstyle{definition}

\theoremstyle{remark}

\newcommand{\ours}{\textsc{MemCon}\xspace}

\fancypagestyle{firstpage}{%
  \fancyhf{}
  \fancyhead[L]{%
    \includegraphics[height=17pt]{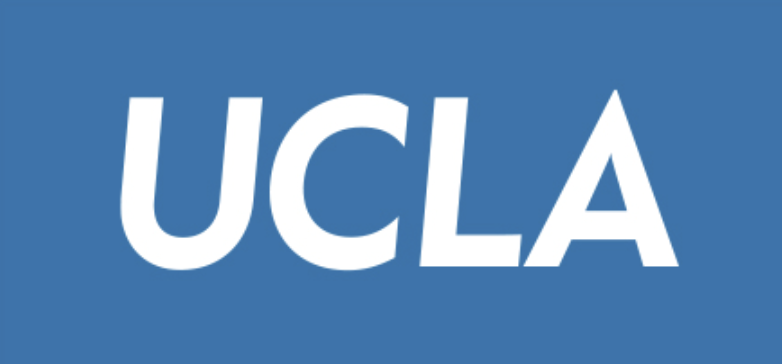}%
    \hspace{0.65em}%
    \includegraphics[height=18pt]{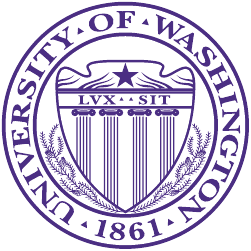}%
    \hspace{0.65em}%
    \includegraphics[height=18pt]{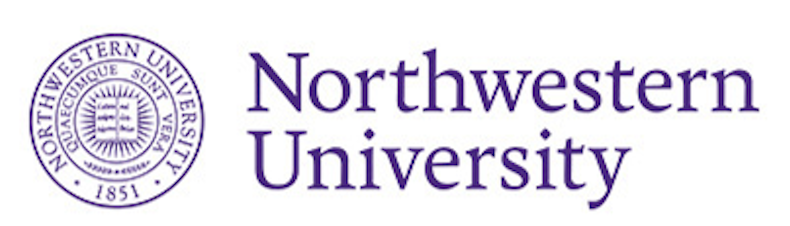}%
  }
  \fancyhead[R]{\textcolor[HTML]{4D4D4D}{\today}}
  \fancyfoot[C]{\thepage}
  \renewcommand{\headrulewidth}{1pt}
  \renewcommand{\footrulewidth}{0pt}
}

\title{Memory as a Controlled Process: Learned Adaptive Memory Management for LLM Agents}

\author{
  \textbf{Eric Hanchen Jiang}\(^{1*}\), \textbf{Zhi Zhang}\(^{1*}\), \textbf{Yuchen Wu}\(^2\), \textbf{Levina Li}\(^1\), \textbf{Dong Liu}\(^1\), \\
  \textbf{Xiao Liang}\(^1\), \textbf{Rui Sun}\(^1\), \textbf{Yubei Li}\(^1\), \textbf{Edward Sun}\(^1\), \textbf{Haozheng Luo}\(^3\), \\
  \textbf{Zhaolu Kang}, \textbf{Aylin Caliskan}\(^2\), \textbf{Kai-Wei Chang}\(^1\), \textbf{Ying Nian Wu}\(^{1\dagger}\) \\[0.8ex]
  \(^1\)University of California Los Angeles, \quad \(^2\)University of Washington, \\
  \(^3\)Northwestern University
}

\begin{document}

\maketitle

\begingroup
\renewcommand\thefootnote{}\footnote{\(^*\)Equal contribution.}
\footnote{\(^\dagger\)Corresponding author: Ying Nian Wu (\texttt{ywu@stat.ucla.edu}).}
\addtocounter{footnote}{-2}
\endgroup

\begin{abstract}
Large Language Models (LLM) agents increasingly rely on external memory systems to accumulate experience across tasks. Yet many existing approaches, from graph-structured memories to reflective insight stores, access memory through \emph{fixed, hand-designed heuristics}. We argue that this static view of memory is a core bottleneck for agentic learning because optimal memory behavior is fundamentally context-dependent. The early stages of tasks benefit from minimal retrieval because memory is sparse; recurring goal types benefit from plan reuse rather than generic nearest-neighbor lookup; stuck agents benefit from re-retrieval with alternative queries; and across long task streams, the memory store itself must be consolidated and pruned to remain useful. We present \ours (\textbf{Mem}ory as a \textbf{Con}trolled Process), a framework that models memory operations as a Markov Decision Process and learns an online policy that adaptively decides \emph{when}, \emph{what}, and \emph{how much} to retrieve, when to inject a distilled plan, and when to consolidate or forget. \ours is \emph{backend-agnostic}: it wraps compatible memory implementations through a common interface, learns from task-by-task binary feedback with no controller pretraining or additional LLM calls for policy selection, and uses a lightweight tabular contextual bandit with UCB exploration that adapts during deployment. In the main comparison across 6 benchmarks, 3 agent frameworks, and 3 LLM backbones, \ours improves average performance over the evaluated memory baselines; on GPT-4.1-mini, it raises six-benchmark mean success by 2.6--4.2 points over G-Memory while reducing mean input-token use by 10.9--24.1\%. Our code is available at \url{https://github.com/ericjiang18/MemCon/}
\end{abstract}

\section{Introduction}
\label{sec:intro}

Large Language Models (LLMs) deployed as autonomous agents have demonstrated remarkable capabilities in interactive environments, multi-step planning, and tool use~\citep{yao2023react,yao2023tree,wang2024survey,xi2023rise}.
A critical enabler of this progress is \emph{memory}, the ability to store experiences from solved tasks and retrieve them on future ones, so that an agent can avoid repeating past mistakes, reuse working strategies, and incrementally accumulate domain knowledge rather than starting from scratch on every episode~\citep{park2023generative,shinn2023reflexion,zhao2024expel,wang2023voyager,gmemory}.


\begin{figure}[t]
  \centering
  \includegraphics[width=\linewidth]{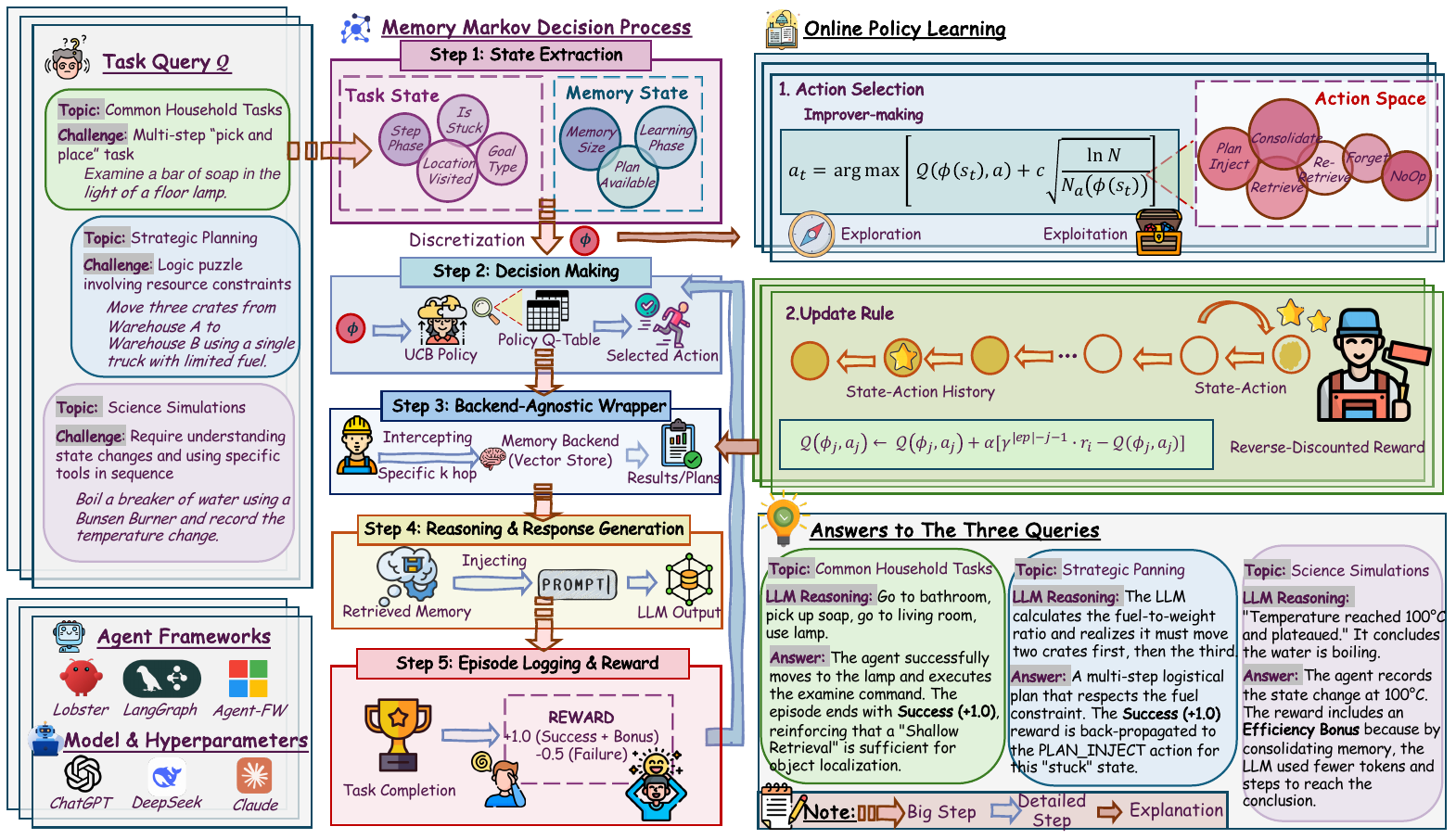}
  \caption{ \small \textbf{Overview of \ours.}
  \textit{(Left)} Task streams from ALFWorld, PDDL, and ScienceWorld are executed through three agent frameworks (Lobster, LangGraph, Agent-FW) sharing one LLM backbone.
  \textit{(Middle)} The \emph{Memory MDP} runs four stages of memory control: extract a compact state $\phi(s)$ from task + memory signals; select an action via the UCB policy over $Q(\phi,a)$; the backend-agnostic wrapper issues retrieval to the inner backend with policy-chosen \texttt{top\_k}/\texttt{insight\_k}/hop; retrieved context (optionally plus an injected success plan) is fed to the LLM, and the episode is scored ($+1$ success, $-0.5$ failure, plus efficiency bonus).
  \textit{(Right)} Online learning: the action space is \{\textsc{Retrieve} (varying depth), \textsc{PlanInject}, \textsc{Re-Retrieve}, \textsc{Consolidate}, \textsc{Forget}, \textsc{NoOp}\}; after each episode the reverse-discounted reward $\gamma^{|\text{ep}|-j-1} r_i$ updates every visited $(\phi_j,a_j)$.}
  \label{fig:memcon-overview}
\end{figure}

Despite rapid progress in \emph{what} to store (trajectories, reflections, insights, skills, graphs, latent tokens~\citep{gmemory,latentmem,packer2023memgpt,chhikara2025mem0,memp,park2023generative}), many existing memory systems treat memory \emph{access} as a static pipeline: a single global retriever is called once per step with a hard-coded \texttt{top}-$k$, a hard-coded graph hop depth, and a hard-coded query template, and uses the same configuration whether memory is empty or contains thousands of trajectories, whether the task is familiar or novel, and whether the agent is progressing or has been executing the same action for five steps. A complementary line of work---most prominently MemGPT~\citep{packer2023memgpt}---makes memory access adaptive by promoting the LLM itself to the role of memory controller (paginating, formulating queries, and deciding when to recall), using LLM inference to select memory operations. \ours occupies a third design point: a \emph{learned but lightweight} controller that is adaptive without any extra LLM calls.
This static-pipeline view can be a bottleneck for agentic memory systems: static memory systems can \emph{either} retrieve too aggressively (inflating context and cost while hurting accuracy on simple queries) \emph{or} too conservatively (missing key reusable plans and reflections), because any single setting is miscalibrated for at least one of the regimes above.

We argue that effective agent memory requires an adaptive control layer sitting on top of the storage backend. Concretely:
(i) \emph{Early tasks} should retrieve \emph{less}, because there is little useful experience to draw on and indiscriminate retrieval only dilutes the prompt.
(ii) \emph{Recurring goal types} should prefer \emph{plan reuse}, replaying a distilled, object-generalized template of a prior success, rather than nearest-neighbor retrieval over raw trajectories.
(iii) \emph{Stuck agents} that repeat actions should trigger \emph{re-retrieval} with an alternative query rather than re-reading the same top-$k$ that already failed to help.
(iv) \emph{Long task streams} require \emph{consolidation and forgetting} so that a growing, noisy memory remains useful rather than overwhelming.
These regimes motivate a policy rather than one global setting; the right memory operation is a function of task progress, memory state, and the agent's learning phase.

We introduce \ours (\textbf{Mem}ory as a \textbf{Con}trolled Process), a framework that reformulates memory access as a \emph{sequential decision problem}. \ours casts the choice of memory operation (\textsc{Retrieve}, \textsc{PlanInject}, \textsc{Re-Retrieve}, \textsc{Consolidate}, \textsc{Forget}, \textsc{NoOp}) together with its parameters (\texttt{top\_k}, \texttt{insight\_k}, graph \texttt{hop}) as actions in a \emph{Memory MDP}, whose state captures both task progress (goal type, step phase, stuck indicator, locations visited) and memory status (size, plan availability, learning phase).
\ours learns a policy over this MDP online, during deployment, using a lightweight tabular contextual bandit with UCB exploration~\citep{auer2002finite,li2010contextual,sutton2018reinforcement}. The policy is warm-started from human-readable priors, updated via reverse-discounted credit assignment from binary task success, and persists to disk between tasks. Crucially, \ours adds \textbf{zero additional LLM calls}: policy selection uses a table lookup and does not invoke the backbone.

\ours is not a new memory store. It is a thin wrapper that intercepts the abstract \texttt{retrieve} and \texttt{store} entry points of any \emph{existing} memory backend and decides how to call them. The interface is designed to support flat vector stores, skill libraries~\citep{wang2023voyager}, summarisation-based memories~\citep{zhong2024memorybank}, latent-memory systems~\citep{latentmem}, and graph-structured memories~\citep{gmemory}, cleanly separating \emph{what} is stored (backend) from \emph{how} it is accessed (controller) while requiring an adapter for backend-specific capabilities.

For the main comparison, we evaluate \ours across 6 benchmarks covering both interactive decision-making (ALFWorld~\citep{shridhar2021alfworld}, PDDL planning, ScienceWorld~\citep{wang2022scienceworld}) and knowledge / web / tool-use QA (TriviaQA~\citep{joshi2017triviaqa}, WebWalkerQA~\citep{wu2024webwalker}, GAIA~\citep{mialon2023gaia}); 3 agent frameworks (Lobster, LangGraph~\citep{park2024langgraph}, Microsoft Agent-Framework); 3 LLM backbones (GPT-4.1-mini, Claude Sonnet-4, DeepSeek-V3.2); and compare against memory baselines spanning vector retrieval (MetaGPT~\citep{hong2024metagpt}, MemoryBank~\citep{zhong2024memorybank}), skill libraries (Voyager~\citep{wang2023voyager}), trajectory summarization (ChatDev~\citep{qian2024chatdev}), generative re-ranking (Generative~\citep{park2023generative}, ExperienceBank), an adapted rule-memory module (OAgent), graph-based hierarchical memory (G-Memory~\citep{gmemory}), and latent memory (LatentMem~\citep{latentmem}). Additional results for GPT-6 Luna and Qwen 3.8 27B appear in Appendix~\ref{app:full-results}.
Across these configurations, \ours achieves the best or near-best task success---e.g., $67.9\%$ ALFWorld on GPT-4.1-mini (Lobster), the top score among all 10 memories evaluated in that cell---and improves six-benchmark mean S/A over G-Memory by 2.6--4.2 points on GPT-4.1-mini while reducing mean input-token consumption by 10.9--24.1\%.

\textbf{Contributions.}
\begin{enumerate}[leftmargin=1.5em, itemsep=2pt]
    \item We formalize agent memory management as a \emph{Memory MDP} and introduce a \emph{backend-agnostic wrapper} that decouples the memory \textbf{control policy} from the memory \textbf{storage backend}, with a small interface for integrating compatible backends (\S\ref{sec:memory-mdp},~\S\ref{sec:wrapper}).
    \item We propose an \emph{online contextual bandit policy} with UCB exploration, warm-start priors, and reverse-discounted credit assignment that learns during deployment using zero controller pretraining and zero extra LLM calls for policy selection, together with augmented memory operations (\S\ref{sec:policy},~\S\ref{sec:augmented}).
    \item We conduct an extensive evaluation of agent memory on six benchmarks and three agent frameworks, showing consistent accuracy gains \emph{and} token savings, including a component ablation that isolates the learned controller's contribution (\S\ref{sec:experiments}).
\end{enumerate}

\section{Related Work}
\label{sec:related}

\paragraph{Memory representations and experience reuse.}
Agent memory stores trajectories, reflections, rules, procedures, or structured relations: Generative Agents~\citep{park2023generative}, MemoryBank~\citep{zhong2024memorybank}, Reflexion~\citep{shinn2023reflexion}, ExpeL~\citep{zhao2024expel}, and CLIN~\citep{majumder2023clin} provide complementary ways to retain experience. Voyager~\citep{wang2023voyager}, Agent Workflow Memory~\citep{wang2024agent}, MemP~\citep{memp}, ProcMEM~\citep{procmem}, and MemSkill~\citep{memskill} emphasize reusable procedures; G-Memory~\citep{gmemory} and LatentMem~\citep{latentmem} organize memory for multi-agent systems. These representations address what experience is available. MemCon focuses on the policy that selects memory operations and their parameters. Our comparisons use adapted memory modules in a shared harness rather than reproductions of every original end-to-end system.

\paragraph{Adaptive retrieval and learned memory control.}
Adaptive access is already an active research direction. MemGPT~\citep{packer2023memgpt} delegates memory operations to an LLM; SimpleMem~\citep{liu2026simplemem} adapts retrieval to query complexity, and xMemory~\citep{hu2026xmemory} expands hierarchical retrieval according to reader uncertainty. Memory-R1~\citep{yan2026memoryr1} fine-tunes memory-management and answer agents with outcome-driven reinforcement learning, while MemAct~\citep{zhang2026memact} learns working-context editing actions. Thus, neither adaptive retrieval nor learning memory operations is unique to MemCon. Our design combines an external cross-task store, execution and memory-state features, and a small online controller trained from terminal task feedback while keeping the LLM frozen. This targets settings where controller fine-tuning or additional LLM inference is undesirable.

\paragraph{Online learning and agent orchestration.}
UCB and contextual bandits~\citep{auer2002finite,li2010contextual,lattimore2020bandit} motivate inexpensive exploration over discrete memory actions. MemCon uses an exponentially weighted terminal-return estimate, rather than a bootstrapped Q-learning target~\citep{sutton2018reinforcement}; the distinction matters for its guarantees (Appendix~\ref{app:theory}). Agent systems such as MetaGPT~\citep{hong2024metagpt}, ChatDev~\citep{qian2024chatdev}, and LangGraph~\citep{park2024langgraph} provide different orchestration structures. Our evaluation varies the runner while keeping the memory-control formulation fixed.

\section{Method}
\label{sec:method}

\begin{figure}[t]
\centering
\includegraphics[width=\linewidth]{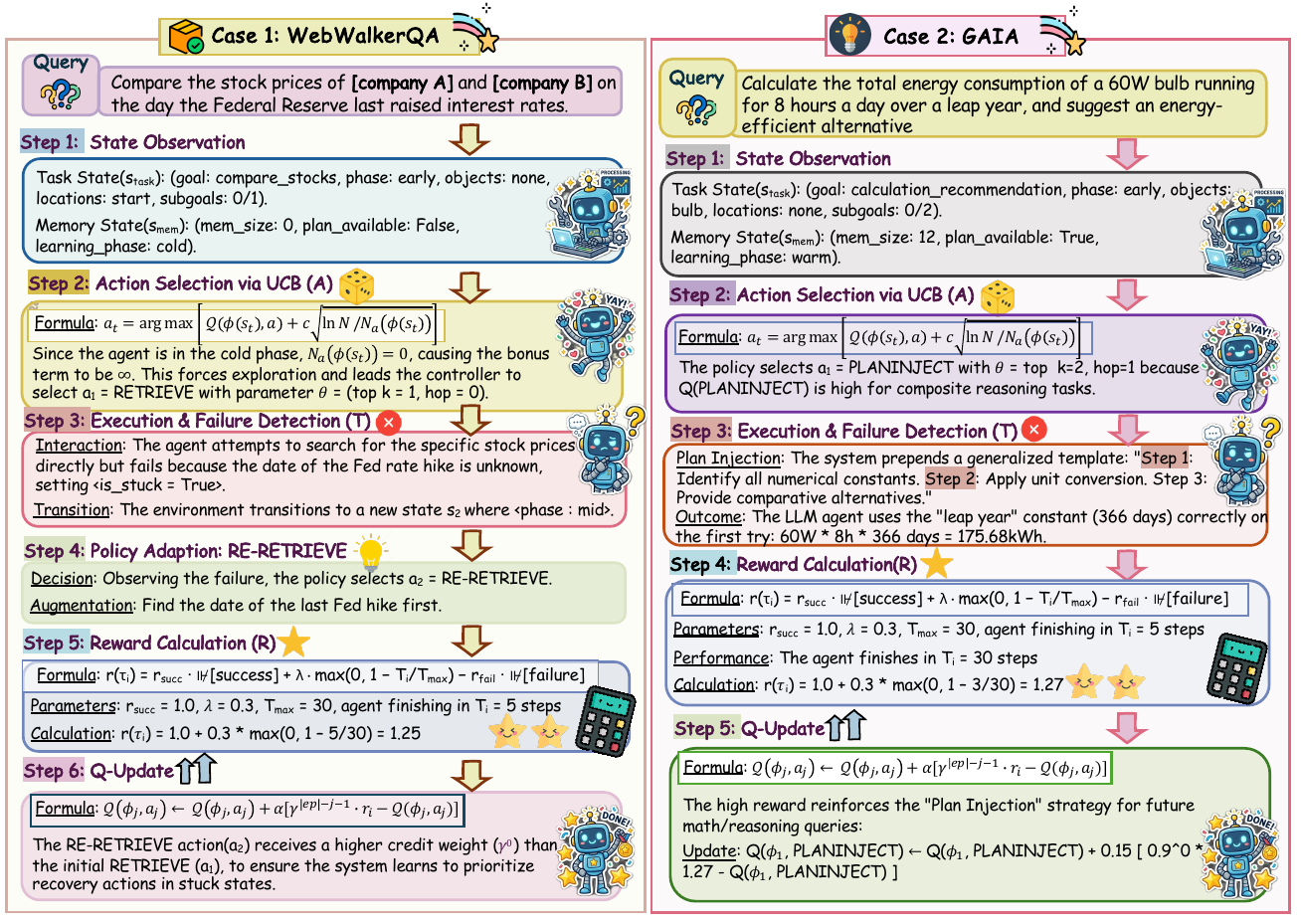}
\caption{\textbf{Illustrative memory-control scenarios.} The left example depicts alternative-query retrieval after a failed attempt; the right depicts plan reuse. These are schematic examples of the intended behavior, not recorded evaluation traces or aggregate evidence of adaptation. In particular, the original QA evaluation used constant interactive-state coordinates.}
\label{fig:case-study}
\end{figure}

This section formalizes each component: the Memory MDP (\S\ref{sec:memory-mdp}), the online policy (\S\ref{sec:policy}), the backend-agnostic wrapper (\S\ref{sec:wrapper}), and two augmented memory operations tailored to long-horizon agentic tasks (\S\ref{sec:augmented}). Figure~\ref{fig:memcon-overview} gives a high-level overview of \ours. 

\subsection{Problem Formulation: Memory MDP}
\label{sec:memory-mdp}

Consider an LLM agent solving tasks $\{\tau_1, \ldots, \tau_N\}$ sequentially. At each task $\tau_i$, the agent interacts with an environment over $T_i$ steps with access to a memory system $\mathcal{M}$ (the \emph{backend}). Most existing memory backends expose two abstract endpoints, a \texttt{retrieve} call (return up to $K$ items relevant to a query) and a \texttt{store} call (write the trajectory of the just-finished task), together with optional maintenance hooks for consolidation and eviction. In standard usage these endpoints are invoked with hard-coded hyperparameters (number of items, search depth, query template, consolidation schedule); we instead model the choice of which endpoint to invoke and with what parameters as a sequential decision problem captured by an MDP $\mathcal{M}_{\text{mem}} = (\mathcal{S}, \mathcal{A}, \mathcal{T}, \mathcal{R}, \gamma)$.

\paragraph{State $\mathcal{S}$.} The state $s = (s^{\text{task}}, s^{\text{mem}})$ fuses task-progress and memory-status signals:
\begin{align}
    s^{\text{task}} &= (\texttt{goal\_type},\; \texttt{step\_phase},\; \texttt{is\_stuck},\; \texttt{objects\_held},\; \texttt{locations}), \\
    s^{\text{mem}}  &= (\texttt{mem\_size},\; \texttt{plan\_available},\; \texttt{learning\_phase}),
\end{align}
where $\texttt{step\_phase}$ bins the step count, $\texttt{is\_stuck}\in\{\text{True},\text{False}\}$ is set when the agent has emitted the same physical action twice in a row, $\texttt{learning\_phase}$ indicates whether the agent has completed enough tasks to trust learned Q-values, and $\texttt{plan\_available}$ records whether a distilled success plan exists for the current goal type. We discretize $s$ into a compact hashable key $\phi(s)$ for tabular learning; the exact bins, thresholds, and the discretization rule are given in Appendix~\ref{app:hparams}. The coarse key keeps policy selection inexpensive; effective state variation depends on the task interface.

\paragraph{Actions $\mathcal{A}$.} Each action $a = (\texttt{op}, \theta)$ specifies a memory operation together with its parameters:
\begin{equation}\label{eq:action-op}
    \texttt{op} \in \{\textsc{Retrieve},\; \textsc{PlanInject},\; \textsc{Re-Retrieve},\; \textsc{Consolidate},\; \textsc{Forget},\; \textsc{NoOp}\};
\end{equation}
\begin{equation}\label{eq:action-theta}
    \theta = (\texttt{top\_k},\; \texttt{insight\_k},\; \texttt{hop}).
\end{equation}
\textsc{Retrieve} returns the top-$\texttt{top\_k}$ items from the backend together with up to $\texttt{insight\_k}$ derived rules (when supported) at search depth $\texttt{hop}$; \textsc{Re-Retrieve} re-issues a retrieval with an alternative-approach query suffix, used to escape repeated-action loops; \textsc{PlanInject} prepends a generalised success plan when one is available (\S\ref{sec:augmented}); \textsc{Consolidate} and \textsc{Forget} call any maintenance hooks the backend exposes (e.g., merging or pruning derived rules) and silently no-op on backends that do not implement them; \textsc{NoOp} skips memory access for the current step. We instantiate $\mathcal{A}$ as a small finite set of representative $(\texttt{op}, \theta)$ configurations spanning shallow-to-deep retrieval plus the maintenance and plan operations; the exact action set used in our experiments is listed in Appendix~\ref{app:hparams} (Table~\ref{tab:app-actions}).

\paragraph{Reward $\mathcal{R}$.} Rewards are observed at task end based on environment feedback:
\begin{equation}\label{eq:reward}
    r(\tau_i) \;=\; r_{\text{succ}}\cdot\mathbf{1}[\text{success}] \;+\; \lambda\cdot\max\!\left(0,\; 1 - T_i/T_{\max}\right) \;-\; r_{\text{fail}}\cdot\mathbf{1}[\text{failure}],
\end{equation}
where $r_{\text{succ}}, r_{\text{fail}}, \lambda$ and the horizon $T_{\max}$ are fixed scalars (values in Appendix~\ref{app:hparams}, Table~\ref{tab:app-hparams}). Successful and more \emph{efficient} trajectories therefore earn larger credit; failed trajectories earn negative feedback that discourages costly or counterproductive memory usage.

\paragraph{Transitions $\mathcal{T}$.} The environment transition is driven by the LLM agent and is opaque to the controller. We use a \emph{contextual-bandit approximation with episode-level feedback}: decisions receive discounted terminal rewards without bootstrapping future action values. This is a tractable controller for a sequential problem, not an exact reduction of the MDP. In particular, maintenance may benefit later tasks without receiving that future credit. Appendix~\ref{app:theory} separates idealized stationary-bandit guarantees from the constant-step-size, evolving-memory implementation.

\subsection{Online Policy Learning}
\label{sec:policy}

We learn $\pi: \mathcal{S} \to \mathcal{A}$ online during deployment, with \textbf{no controller pretraining} and \textbf{no additional LLM calls for policy selection}.

\paragraph{Action selection via UCB.} At each memory decision point we apply the Upper Confidence Bound rule~\citep{auer2002finite,li2010contextual}, with the exploration bonus computed using the \emph{state-specific} visit count:
\begin{equation}\label{eq:ucb-rule-main}
    a_t \;=\; \arg\max_{a\in\mathcal{A}} \left[\; Q(\phi(s_t), a) \;+\; c\sqrt{\frac{\ln N(\phi(s_t))}{N_a(\phi(s_t))}} \;\right],
\end{equation}
where $c$ is the UCB exploration coefficient (Appendix~\ref{app:hparams}, Table~\ref{tab:app-hparams}), $N(\phi(s_t)) := \sum_{a\in\mathcal{A}} N_a(\phi(s_t))$ is the total number of decisions ever taken at state key $\phi(s_t)$, and $N_a(\phi(s_t))$ is the count of action $a$ at that state. Unvisited actions receive an $\infty$ bonus, forcing the first $|\mathcal{A}|$ visits to each state to span all actions; warm-start priors (below) break ties among such $\infty$ bonuses to give a sensible initial ordering. This defines the state-specific selection rule; the stationary sample-mean comparison in Appendix~\ref{app:theory} is analyzed separately.

\paragraph{Warm-start priors.} Before any learning, Q-values are initialized with interpretable priors that encode mild domain knowledge: \textsc{Retrieve} and \textsc{PlanInject} receive positive priors (retrieval is usually helpful, plans are useful when available), \textsc{Re-Retrieve} a smaller positive prior (useful only when stuck), \textsc{Consolidate} a neutral prior, and \textsc{Forget} and \textsc{NoOp} mildly negative priors. Numerical values are listed in Appendix~\ref{app:hparams} (Table~\ref{tab:app-hparams}). These priors provide an initial preference before sufficient task feedback is available.

\paragraph{Credit assignment.} After each task $\tau_i$, all $(s,a)$ pairs visited during the episode are updated with a reverse-discounted Monte-Carlo return, so that actions taken closer to the terminal outcome receive stronger credit:
\begin{equation}\label{eq:q-update}
    Q(\phi_j, a_j) \;\leftarrow\; Q(\phi_j, a_j) \;+\; \alpha\left[\gamma^{|\text{ep}|-j-1}\cdot r_i \;-\; Q(\phi_j, a_j)\right],
\end{equation}
where $\alpha$ is the step size, $\gamma\in(0,1)$ is the within-episode discount, $|\text{ep}|$ is the total number of memory decisions executed during episode $\tau_i$, and $j\in\{0,1,\dots,|\text{ep}|-1\}$ is the chronological index of the decision being credited (so the last decision receives the undiscounted $r_i$ and earlier decisions receive geometrically attenuated credit). Numerical values for $\alpha$ and $\gamma$ are in Appendix~\ref{app:hparams}. The Q-table is persisted to disk every few updates, so learning carries across runs and task streams.

\subsection{Backend-Agnostic Wrapper and Augmented Operations}
\label{sec:wrapper}\label{sec:augmented}

\ours is a thin wrapper around \emph{any} memory backend $\mathcal{M}_{\text{inner}}$ that exposes a two-method interface: \texttt{retrieve} (return $K$ items for a query) and \texttt{store} (write the finished trajectory with its success label); an optional \texttt{maintain} hook unlocks the consolidation and eviction actions, and is silently no-op'd otherwise. On \texttt{retrieve}, the wrapper builds $s_t$, invokes the policy to pick $(\texttt{op}^*,\theta^*)$, calls $\mathcal{M}_{\text{inner}}$ with those parameters (unsupported parameters are dropped), and optionally augments the result with one of two augmented operations (described below); on \texttt{store}, it computes the episode reward, updates Q-values via Eq.~\eqref{eq:q-update}, updates the plan index, and delegates persistence to $\mathcal{M}_{\text{inner}}$. This interface separates policy selection from storage and can be adapted to flat vector stores~\citep{hong2024metagpt,zhong2024memorybank}, skill libraries~\citep{wang2023voyager}, summarisation-based memories~\citep{qian2024chatdev}, latent-memory systems~\citep{latentmem}, and graph-structured memories~\citep{gmemory}. The two augmented operations target failure modes in long-horizon interactive tasks: \textbf{generalised plan injection} (a learnable MDP action, \textsc{PlanInject}) extracts the action sequence of a successful task of type $g$, replaces instance-specific identifiers (e.g.\ \texttt{``shelf 3''}\,$\to$\,\texttt{[shelf]}) via a regex rewriter, stores the resulting template under $g$ in a lightweight JSON index, and on future tasks of type $g$ prepends the template to the backend's retrieval output (the rewriter degrades to identity on new vocabularies); \textbf{goal decomposition} (a deterministic heuristic, \emph{not} an MDP action) uses a task-specific rule to handle composite tasks requiring the same primitive twice with two objects (e.g., ``put two cellphones on the desk'') by injecting ``complete all steps for object 1 first, then repeat for object 2'' and additionally retrieving the single-object template when one exists. Both augmentations concatenate text to retrieved context; their parsers and task-type mappings may require domain adaptation; the component ablation (\S\ref{sec:ablation}, Table~\ref{tab:component-ablation}) shows that the learned controller alone already accounts for the majority of \ours's gains, making both augmentations useful but optional.

\begin{table}[t]
\centering
\caption{\textbf{Main results with GPT-4.1-mini.} S/A (\%) across six benchmarks and three runners. Eight adapted memory baselines and Empty are compared with MemCon on G-Memory. \textbf{Bold}: best; \underline{underline}: second-best distinct value in each runner block, including ties. Avg. is the unweighted mean recomputed from the six displayed scores. SciWorld, Trivia, and WebQA denote ScienceWorld, TriviaQA, and WebWalkerQA. Other backbones and token costs appear in Appendix~\ref{app:full-results}.}
\label{tab:results-gpt41mini}
\small
\setlength{\tabcolsep}{3.5pt}
\resizebox{\linewidth}{!}{%
\begin{tabular}{llrrrrrrr}
\toprule
\rowcolor{gray!10} && \multicolumn{3}{c}{\textbf{Interactive}} & \multicolumn{3}{c}{\textbf{QA / tool use}} & \\
\cmidrule(lr){3-5}\cmidrule(lr){6-8}
\rowcolor{gray!10} \textbf{Runner} & \textbf{Memory} & \textbf{ALFWorld} & \textbf{PDDL} & \textbf{SciWorld} & \textbf{Trivia} & \textbf{WebQA} & \textbf{GAIA} & \textbf{Avg.} \\
\midrule
\multirow{10}{*}{Lobster} & Empty & 43.3 & \underline{33.3} & 28.0 & 69.5 & 17.9 & 20.6 & 35.4 \\
 & G-Memory & \underline{59.7} & \underline{33.3} & 34.0 & 66.5 & \underline{19.0} & \underline{21.2} & \underline{39.0} \\
 & MetaGPT & 55.2 & \underline{33.3} & 33.0 & \textbf{71.5} & 18.2 & 16.4 & 37.9 \\
 & Voyager & 53.0 & \underline{33.3} & 31.0 & \textbf{71.5} & 18.4 & 16.4 & 37.3 \\
 & Generative & 54.5 & \underline{33.3} & \underline{36.0} & 69.5 & 18.7 & 18.2 & 38.4 \\
 & ChatDev & 32.8 & 28.3 & 24.0 & 69.5 & 17.7 & 18.2 & 31.8 \\
 & MemoryBank & 57.5 & \underline{33.3} & 28.0 & 69.5 & 16.3 & 17.6 & 37.0 \\
 & OAgent & 48.5 & 28.3 & 34.0 & \underline{70.5} & 17.4 & 19.4 & 36.4 \\
 & ExpBank & 59.0 & \underline{33.3} & 30.0 & 69.0 & 17.8 & 20.0 & 38.2 \\
\rowcolor{mygreen}  & \ours & \textbf{67.9} & \textbf{35.0} & \textbf{38.0} & \textbf{71.5} & \textbf{20.6} & \textbf{22.4} & \textbf{42.6} \\
\midrule
\multirow{10}{*}{LangGraph} & Empty & 31.3 & 30.0 & 39.0 & 67.5 & 18.6 & 19.4 & 34.3 \\
 & G-Memory & \textbf{68.7} & 28.3 & 38.0 & \underline{69.0} & \textbf{20.9} & \underline{20.0} & \underline{40.8} \\
 & MetaGPT & 57.5 & \underline{38.3} & 35.0 & \underline{69.0} & 18.3 & 18.2 & 39.4 \\
 & Voyager & \underline{59.7} & 33.3 & \underline{40.0} & 67.5 & 17.5 & \underline{20.0} & 39.7 \\
 & Generative & 56.0 & 35.0 & 36.0 & 66.0 & 18.0 & \underline{20.0} & 38.5 \\
 & ChatDev & 39.6 & 33.3 & 27.0 & 67.5 & 18.1 & 18.8 & 34.1 \\
 & MemoryBank & 57.5 & 33.3 & 29.0 & \textbf{70.0} & 18.7 & 18.8 & 37.9 \\
 & OAgent & 53.7 & 30.0 & 37.0 & 68.0 & 18.1 & 18.8 & 37.6 \\
 & ExpBank & 55.2 & 35.0 & \textbf{41.0} & 68.0 & 18.2 & \underline{20.0} & 39.6 \\
\rowcolor{mygreen}  & \ours & \textbf{68.7} & \textbf{40.0} & \underline{40.0} & \underline{69.0} & \underline{20.2} & \textbf{22.4} & \textbf{43.4} \\
\midrule
\multirow{10}{*}{Agent-FW} & Empty & 38.1 & 33.3 & 26.0 & \underline{69.0} & 18.1 & 16.4 & 33.5 \\
 & G-Memory & \underline{70.2} & 33.3 & 28.0 & 66.0 & \underline{19.4} & \underline{20.0} & \underline{39.5} \\
 & MetaGPT & 55.2 & 30.0 & \textbf{38.0} & 68.5 & 17.9 & 18.8 & 38.1 \\
 & Voyager & 54.5 & \underline{36.7} & 36.0 & 67.5 & 17.1 & 19.4 & 38.5 \\
 & Generative & 50.0 & 35.0 & \underline{37.0} & 68.0 & 18.3 & 18.8 & 37.9 \\
 & ChatDev & 32.1 & 33.3 & 31.0 & 66.5 & 17.8 & 19.4 & 33.4 \\
 & MemoryBank & 59.0 & 33.3 & 30.0 & \textbf{69.5} & 18.7 & 16.4 & 37.8 \\
 & OAgent & 63.4 & 30.0 & \underline{37.0} & \underline{69.0} & 18.1 & 18.2 & 39.3 \\
 & ExpBank & 52.2 & \underline{36.7} & 35.0 & 66.5 & 17.5 & 17.0 & 37.5 \\
\rowcolor{mygreen}  & \ours & \textbf{71.0} & \textbf{40.0} & \textbf{38.0} & \underline{69.0} & \textbf{20.9} & \textbf{23.0} & \textbf{43.7} \\
\bottomrule
\end{tabular}}
\end{table}

\paragraph{Portability and effective actions.}
Backend-agnostic refers to the interface contract, not to identical functionality or demonstrated gains on every store. The main evaluation wraps G-Memory; Appendix~\ref{app:rebuttal} additionally evaluates five backend pairings. Unsupported parameters can collapse distinct presets into the same effective operation; independently exploring such duplicate arms can waste samples. Backends also differ in whether maintenance calls invoke an LLM. The zero-call claim concerns the controller itself, not the underlying retrieval, summarization, or maintenance implementation.

\section{Experiments}
\label{sec:experiments}
We evaluate whether learned memory control improves task success, how its gains relate to token use, and how much improvement comes from the controller versus augmented operations.

\subsection{Setup}
\label{sec:setup}
We evaluate interactive decision-making on \textbf{ALFWorld}~\citep{shridhar2021alfworld}, \textbf{PDDL planning}, and \textbf{ScienceWorld}~\citep{wang2022scienceworld}, and QA/web/tool use on \textbf{TriviaQA}~\citep{joshi2017triviaqa}, \textbf{WebWalkerQA}~\citep{wu2024webwalker}, and \textbf{GAIA}~\citep{mialon2023gaia}. The main comparison uses three runners (Lobster, LangGraph, Microsoft Agent-Framework) and three backbones (GPT-4.1-mini, Claude Sonnet-4, DeepSeek-V3.2), giving 54 runner--backbone--benchmark configurations. Benchmark sizes, adapted baseline mechanisms, and prompts appear in Appendices~\ref{app:setup} and~\ref{app:prompts}.

\paragraph{Comparisons and metrics.}
The reported MemCon configuration wraps \textbf{G-Memory}. Comparing MemCon with G-Memory therefore holds the inner storage design fixed; comparisons with the other memory modules also change the backend. Table~\ref{tab:results-gpt41mini} includes eight adapted memory baselines and the Empty control; the additional backbone tables also include LatentMem. These are shared-interface adaptations, not replications of the original systems. We report success rate / answer accuracy (\textbf{S/A}, \%) and average LLM-input tokens per task (\textbf{Tok}). Each main-table cell is a single online deployment run. Appendix~\ref{app:rebuttal} adds an independent-stream summary; small main-table differences remain descriptive.

\paragraph{Online protocol.}
Memory and the policy accumulate experience across tasks, so these scores measure sequential adaptation rather than independent, identically distributed trials with a frozen store. Policy hyperparameters $(\alpha,\gamma,c)$ are selected per backbone on a held-out 30-task ALFWorld subset; priors, reward constants, and discretization thresholds are shared (Appendix~\ref{app:hparams}). In the original QA configuration, interactive features are constant and the effective state reduces to memory size, plan availability, and learning phase. QA results therefore test memory-state adaptation, without establishing that stuck-state recovery operates in those runs.

\subsection{Main Results}
\label{sec:main-results}

\paragraph{Accuracy across frameworks and backbones.}
On GPT-4.1-mini, MemCon is best or tied best in \textbf{14/18} cells: 6/6 under Lobster, 3/6 under LangGraph, and 5/6 under Agent-FW (11 strict wins). Its six-benchmark mean S/A is 42.6, 43.4, and 43.7, respectively, compared with 39.0, 40.8, and 39.5 for G-Memory. Lobster--ALFWorld improves from 59.7 to 67.9\%; Agent-FW--PDDL improves from 33.3 to 40.0\%. The exceptions also matter: under LangGraph, MemCon trails the best adapted baseline on ScienceWorld, TriviaQA, and WebWalkerQA. These results favor adaptive control on average without implying dominance in every setting.

The same pattern extends to Claude Sonnet-4 and DeepSeek-V3.2 (Appendix~\ref{app:full-results}). MemCon leads \textbf{15/18} cells on Sonnet-4 and \textbf{15/18} on DeepSeek-V3.2. On Sonnet-4, the three deficits to the best baseline are 0.2--0.4 points; GAIA improves over the strongest alternative by 0.8, 2.6, and 2.8 points across runners. On DeepSeek-V3.2, MemCon leads all nine interactive cells; the three exceptions are GAIA, where LatentMem leads by 0.3--0.4 points. Across these three backbones, the reported scores are best or tied best in 44/54 cells. Independent-stream results for selected benchmarks are reported in Appendix~\ref{app:rebuttal}.

\begin{figure}[t]
  \centering
  \includegraphics[width=1.0\linewidth]{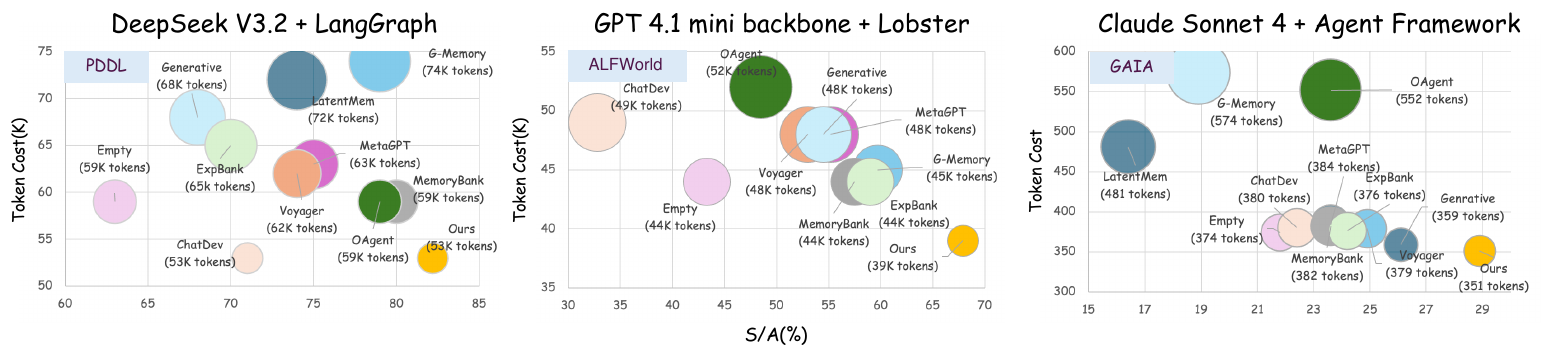}
  \caption{ \small \textbf{Token cost vs.\ task success across memories.}
  Each panel plots, for one (framework, benchmark) pair, every memory baseline as a bubble whose horizontal position is mean S/A (\%), vertical position is mean per-task input tokens (in thousands or units), and area is proportional to per-task token cost.
  \ours illustrates a favorable accuracy--cost trade-off in the displayed configurations. Table~\ref{tab:app-tokens} reports the full GPT-4.1-mini token comparison; savings are relative to G-Memory and are not universal across baselines. }
  \label{fig:tokens-pareto}
\end{figure}

\paragraph{Accuracy and token use.}
MemCon reduces input-token use relative to the same G-Memory backend in all 18 GPT-4.1-mini cells (Table~\ref{tab:app-tokens}). On Lobster--ALFWorld, accuracy rises by 8.2 points while input tokens fall from 45K to 39K (13.3\%); on Agent-FW--ALFWorld, accuracy rises from 70.2 to 71.0\% while tokens fall from 43K to 37K (14.0\%). Averaging the six displayed benchmark costs gives approximate reductions of 10.9\%, 11.0\%, and 24.1\% for Lobster, LangGraph, and Agent-FW. These are descriptive averages over heterogeneous tasks, not pooled workload costs. MemCon is not always the cheapest memory: several simpler QA baselines use fewer tokens, and some no-memory runs are cheaper. The evidence supports improved efficiency relative to G-Memory rather than a universal minimum-cost claim.

\paragraph{Generalization to QA and tool use.}
On GPT-4.1-mini GAIA, MemCon reaches 22.4--23.0\% versus 20.0--21.2\% for G-Memory, while TriviaQA results are tied with or slightly below the best adapted baseline in several settings. For 165 GAIA tasks, a 1.2-point difference corresponds to approximately two additional correct answers. Thus, the QA evidence is consistent with useful control over retrieval even with a reduced state, but does not establish a large or statistically stable gain. No backbone weights are updated during these runs.

\subsection{Component Analysis and Sensitivity}
\label{sec:ablation}

\begin{figure}[t]
\centering
\includegraphics[width=\linewidth]{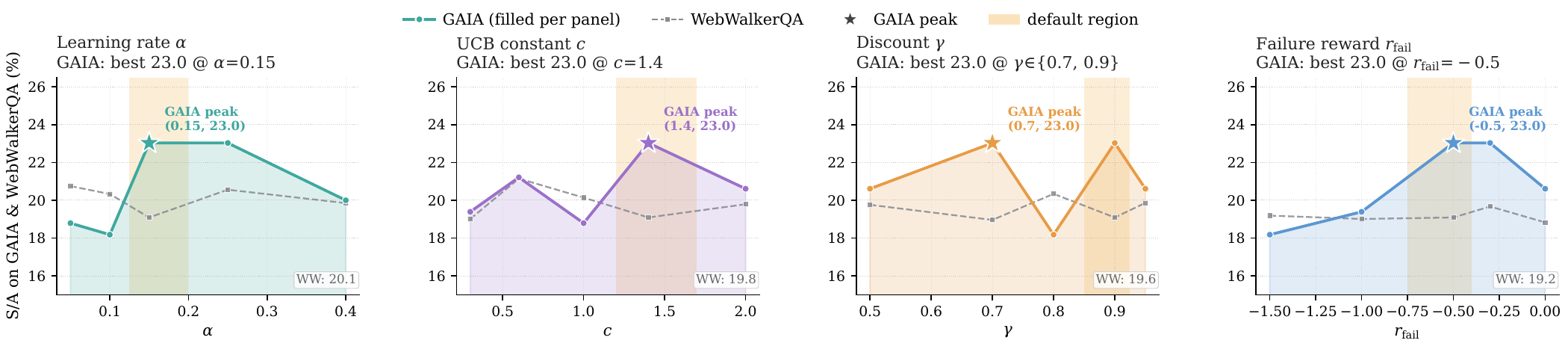}
\caption{\textbf{Single-knob sensitivity (continuous policy hyperparameters).} Each panel sweeps one parameter---learning rate $\alpha$, UCB coefficient $c$, discount $\gamma$, or the signed failure reward $-r_{\mathrm{fail}}$---with the others fixed. Solid colored lines show GAIA S/A; gray dashed lines show WebWalkerQA. Shading marks the default region and stars mark GAIA peaks. Backbone: GPT-4.1-mini; runner: Lobster.}
\label{fig:hp-ablation-top}
\medskip
\begin{minipage}[t]{0.55\linewidth}
\vspace{0pt}
\includegraphics[width=\linewidth]{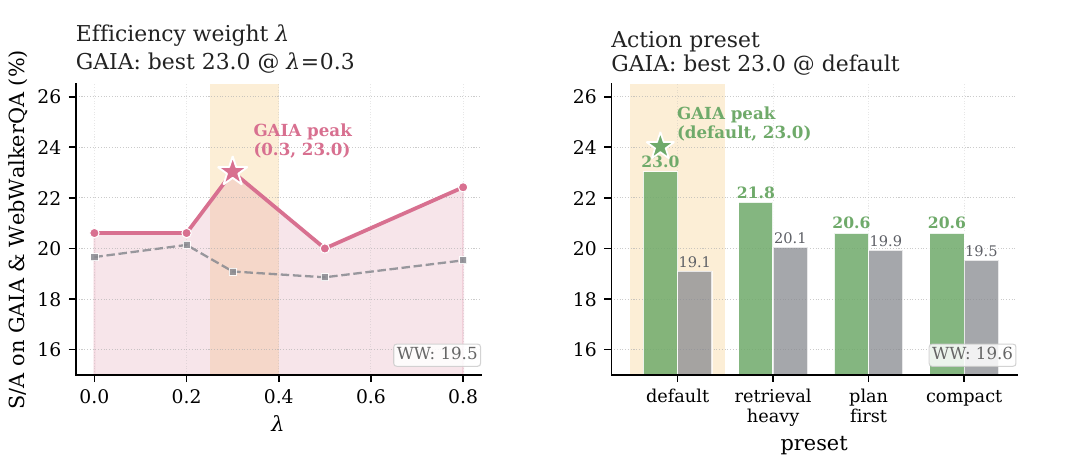}
\captionof{figure}{\textbf{Efficiency weight and action preset.} Left: varying $\lambda$ favors the default $0.3$; removing or over-weighting the bonus reduces GAIA accuracy. Right: the default nine-action preset has the highest GAIA score; retrieval-heavy has the highest WebWalkerQA score.}
\label{fig:hp-ablation-bottom}
\end{minipage}\hfill
\begin{minipage}[t]{0.43\linewidth}
\vspace{0pt}
\centering
\small
\setlength{\tabcolsep}{2pt}
\resizebox{\linewidth}{!}{%
\begin{tabular}{lrr}
\toprule
\textbf{Variant} & \textbf{ALFWorld} & \textbf{GAIA}\\
\midrule
Static backend & 59.7 & 21.2\\
\quad + learned UCB & 64.9 & 22.4\\
\quad + plan injection ($g$) & 66.4 & 22.4\\
\quad + goal decomposition & 67.2 & 22.4\\
\rowcolor{mygreen} \quad + all (= \ours) & \textbf{67.9} & \textbf{22.4}\\
\bottomrule
\end{tabular}}
\captionof{table}{\textbf{Component ablation on GPT-4.1-mini.} Lobster--ALFWorld success and GAIA S/A (\%) for all five reported variants. The learned controller supplies the largest single gain; the full configuration further improves ALFWorld.}
\label{tab:component-ablation}
\end{minipage}
\end{figure}

Table~\ref{tab:component-ablation} reports the complete five-variant component ablation. The learned UCB controller improves ALFWorld from 59.7 to 64.9\%, accounting for 5.2 of the full 8.2-point gain (about 63\%). The plan-injection and goal-decomposition configurations reach 66.4 and 67.2\%, and the full system reaches 67.9\%. On GAIA, all controller variants achieve 22.4\%, compared with 21.2\% for the static backend. Thus, learned access control contributes without relying on the augmented operations, while the full configuration offers additional gains on interactive tasks.

Figures~\ref{fig:hp-ablation-top} and~\ref{fig:hp-ablation-bottom} jointly show all six sensitivity panels. Intermediate policy settings perform well; removing the efficiency bonus or giving it excessive weight reduces accuracy by roughly two points. These local sweeps hold other parameters fixed and complement the component comparison. 

\paragraph{Additional experiments.}
Appendix~\ref{app:rebuttal} extends the evaluation to five memory backends and additional adaptive-access baselines, examining how the controller performs beyond the main G-Memory pairing.
The 400-task experiments study memory maintenance over longer deployments, where consolidation and forgetting can influence subsequent tasks.
Additional analyses of state design, transfer, and initialization examine the sensitivity and transferability of the controller.
A five-seed evaluation across independent task streams complements the single-deployment results by assessing variation across runs.
Together, these studies complement the main accuracy and token-cost comparisons with a broader assessment of backend portability, long-horizon behavior, and deployment variability.

\section{Conclusion}
\label{sec:conclusion}
We introduced MemCon, a framework that treats agent memory as a controlled decision process. A lightweight online UCB controller chooses memory operations and parameters from task and memory state while keeping the LLM frozen. Across the reported benchmarks, frameworks, and backbones, MemCon improves average success. On GPT-4.1-mini, it also reduces input-token use relative to G-Memory. The component analysis attributes most of the ALFWorld gain to learned control. Together with the additional backend and long-stream experiments, these findings motivate separating memory storage from access policy and improving credit assignment for long-term maintenance.

\bibliographystyle{arxiv_version}
\bibliography{references}
\clearpage
\appendix
\raggedbottom
\section{Experimental Setup Details}
\label{app:setup}

We detail the reported benchmark subsets, adapted memory modules, and tuning protocol.

\subsection{Benchmarks}
\label{app:setup-benchmarks}

\paragraph{Interactive decision-making (long-horizon, multi-step).}
\textbf{ALFWorld}~\citep{shridhar2021alfworld} contains 134 household tasks across 6 types (\textit{put, clean, heat, cool, examine, puttwo}); it couples a TextWorld front-end with embodied AI2-THOR back-ends.
\textbf{PDDL Planning} consists of 100 classical planning tasks spanning \textit{blocksworld, barman, gripper, tyreworld}, closely related to PlanBench~\citep{valmeekam2023planbench,liu2023llmplus,helmert2006fast}.
\textbf{ScienceWorld}~\citep{wang2022scienceworld} contains 100 elementary-science experimental tasks (boiling water, melting ice, finding animals by property, etc.).

\paragraph{Knowledge / web / tool-use QA.}
\textbf{TriviaQA}~\citep{joshi2017triviaqa} (200 questions, short-answer open-domain QA);
\textbf{WebWalkerQA}~\citep{wu2024webwalker} (200 multi-hop web traversal questions requiring site navigation and information extraction);
\textbf{GAIA}~\citep{mialon2023gaia} (165 general-AI-assistant questions spanning tool use, web search, and document reading).

\subsection{Agent Frameworks}
\label{app:setup-frameworks}

The runners share task loaders, tools, and evaluation protocols. \textbf{Lobster} calls an OpenAI-compatible chat API directly; \textbf{LangGraph}~\citep{park2024langgraph} provides graph-structured workflows; \textbf{Microsoft Agent-Framework} provides pipeline orchestration. Interactive tasks use a 30-step horizon; QA uses benchmark-default budgets. Temperature and tool protocols are shared across memory comparisons.

\subsection{Memory Baselines}
\label{app:setup-baselines}

The table labels identify adapted modules on a shared retrieve/store interface:
\textbf{MetaGPT} uses vector retrieval over trajectories;
\textbf{Voyager} summarizes trajectories before cosine-similarity retrieval;
\textbf{Generative} uses LLM relevance re-ranking;
\textbf{ChatDev} summarizes phases every 10 steps without cross-task retrieval;
\textbf{MemoryBank} adds temporal decay;
\textbf{OAgent} induces rules by comparing successful and failed trials;
\textbf{ExperienceBank} uses LLM-scored experience retrieval;
and \textbf{G-Memory} combines vector retrieval, graph traversal, and scored insight rules.
The additional backbone tables also include a \textbf{LatentMem}-labeled module. These labels denote experimental adaptations rather than verified reproductions of the cited systems, particularly for OAgent and LatentMem, whose original descriptions do not establish implementation equivalence. The \textbf{Empty} control disables memory. MemCon wraps G-Memory, making this pair the controlled comparison that holds the storage design fixed.

\subsection{Hyperparameter Selection Protocol}
\label{app:setup-hparams}

MemCon uses the default nine-action set, warm-start priors, and the reward in Eq.~\eqref{eq:reward}. The per-backbone policy settings $(\alpha,\gamma,c)$ are chosen by grid search on a held-out 30-task ALFWorld subset (Lobster); priors, reward constants, and bins are shared across backbones. Appendix~\ref{app:hparams} lists the values. The Q-table and plan index persist across tasks. The main comparison tables use one online deployment stream per configuration; Appendix~\ref{app:rebuttal} adds the five-seed summary.

\clearpage
\section{Per-Backbone S/A Tables and Token Cost}
\label{app:full-results}

This appendix reports the full per-backbone S/A tables for Claude Sonnet-4, DeepSeek-V3.2, GPT-6 Luna (Table~\ref{tab:results-gpt6luna}), and Qwen 3.8 27B (Table~\ref{tab:results-qwen38}) and provides the available GPT-4.1-mini average per-task token costs (Tok).

\subsection{Claude Sonnet-4 (S/A only)}

\begin{table}[H]
\centering
\caption{\textbf{Claude Sonnet-4 results.} S/A (\%) on the same six benchmarks. This table adds the adapted LatentMem-labeled module to the baselines in Table~\ref{tab:results-gpt41mini}. \textbf{Bold}: best; \underline{underline}: second-best distinct value, including ties. Avg. is recomputed from displayed scores.}
\label{tab:results-sonnet}
\small
\setlength{\tabcolsep}{3.5pt}
\resizebox{\linewidth}{!}{%
\begin{tabular}{llrrrrrrr}
\toprule
\rowcolor{gray!10} && \multicolumn{3}{c}{\textbf{Interactive}} & \multicolumn{3}{c}{\textbf{QA / tool use}} & \\
\cmidrule(lr){3-5}\cmidrule(lr){6-8}
\rowcolor{gray!10} \textbf{Runner} & \textbf{Memory} & \textbf{ALFWorld} & \textbf{PDDL} & \textbf{SciWorld} & \textbf{Trivia} & \textbf{WebQA} & \textbf{GAIA} & \textbf{Avg.} \\
\midrule
\multirow{11}{*}{Lobster} & Empty & 14.2 & 18.0 & 19.0 & 81.5 & 17.1 & 24.2 & 29.0 \\
 & G-Memory & \underline{31.3} & 58.0 & 25.0 & 82.0 & 16.3 & 20.6 & 38.9 \\
 & MetaGPT & 6.7 & 8.0 & 12.0 & \underline{82.5} & 15.9 & 24.9 & 25.0 \\
 & Voyager & 10.4 & 14.0 & 10.0 & 81.5 & 16.7 & 24.9 & 26.3 \\
 & Generative & 6.7 & 15.0 & 15.0 & 81.5 & 16.4 & 25.4 & 26.7 \\
 & ChatDev & 11.9 & 17.0 & 14.0 & 81.5 & 16.6 & \underline{26.7} & 28.0 \\
 & MemoryBank & 9.0 & 12.0 & 8.0 & 81.5 & 16.5 & 22.4 & 24.9 \\
 & OAgent & 28.4 & 42.0 & 27.0 & 81.5 & \underline{17.4} & 24.2 & 36.8 \\
 & ExpBank & 5.2 & 11.0 & 5.0 & 81.5 & 16.2 & 23.6 & 23.8 \\
 & LatentMem & 30.6 & \underline{67.0} & \textbf{82.0} & 81.5 & 15.8 & 17.6 & \underline{49.1} \\
\rowcolor{mygreen}  & \ours & \textbf{32.3} & \textbf{68.9} & \underline{81.6} & \textbf{84.4} & \textbf{19.4} & \textbf{27.5} & \textbf{52.4} \\
\midrule
\multirow{11}{*}{LangGraph} & Empty & 17.2 & 15.0 & 20.0 & 82.5 & 15.9 & \underline{25.4} & 29.3 \\
 & G-Memory & 46.3 & 33.0 & 37.0 & 81.5 & 17.4 & 19.4 & 39.1 \\
 & MetaGPT & 16.4 & 24.0 & 21.0 & \textbf{83.0} & 15.4 & 23.0 & 30.5 \\
 & Voyager & 22.4 & 22.0 & 33.0 & 82.5 & 17.1 & \underline{25.4} & 33.7 \\
 & Generative & 17.9 & 26.0 & 31.0 & 82.0 & 16.6 & 21.8 & 32.6 \\
 & ChatDev & 12.7 & 28.0 & 20.0 & 82.5 & 16.1 & 22.4 & 30.3 \\
 & MemoryBank & 26.9 & 25.0 & 39.0 & 82.0 & 15.8 & 22.4 & 35.2 \\
 & OAgent & 53.0 & 33.0 & 58.0 & 80.5 & 16.9 & 17.6 & 43.2 \\
 & ExpBank & 17.9 & 22.0 & 21.0 & 82.0 & 16.3 & 23.0 & 30.4 \\
 & LatentMem & \underline{53.7} & \underline{61.0} & \underline{68.0} & 82.5 & \underline{18.0} & 23.0 & \underline{51.0} \\
\rowcolor{mygreen}  & \ours & \textbf{55.4} & \textbf{63.2} & \textbf{69.3} & \underline{82.8} & \textbf{19.2} & \textbf{28.0} & \textbf{53.0} \\
\midrule
\multirow{11}{*}{Agent-FW} & Empty & 12.7 & 21.0 & 17.0 & \underline{83.0} & 17.1 & 21.8 & 28.8 \\
 & G-Memory & 20.9 & 23.0 & 33.0 & 81.5 & \textbf{17.5} & 18.8 & 32.5 \\
 & MetaGPT & 15.7 & 19.0 & 19.0 & 81.0 & \textbf{17.5} & 23.6 & 29.3 \\
 & Voyager & 10.4 & 14.0 & 14.0 & 81.0 & 9.5 & 24.9 & 25.6 \\
 & Generative & 20.2 & 20.0 & 26.0 & 82.0 & \underline{17.2} & \underline{26.1} & 31.9 \\
 & ChatDev & 12.7 & 16.0 & 19.0 & 82.0 & 15.6 & 22.4 & 28.0 \\
 & MemoryBank & 11.9 & 20.0 & 20.0 & 81.5 & \textbf{17.5} & 23.6 & 29.1 \\
 & OAgent & 23.9 & 30.0 & 28.0 & 81.5 & 10.2 & 23.6 & 32.9 \\
 & ExpBank & 6.0 & 6.0 & 7.0 & 82.0 & 16.4 & 24.2 & 23.6 \\
 & LatentMem & \underline{58.2} & \underline{70.0} & \underline{66.0} & 81.5 & 15.1 & 16.4 & \underline{51.2} \\
\rowcolor{mygreen}  & \ours & \textbf{60.6} & \textbf{70.7} & \textbf{67.1} & \textbf{84.3} & 17.1 & \textbf{28.9} & \textbf{54.8} \\
\bottomrule
\end{tabular}}
\end{table}

\clearpage
\subsection{DeepSeek-V3.2 (S/A only)}

\begin{table}[H]
\centering
\caption{\textbf{DeepSeek-V3.2 results.} S/A (\%) on the same six benchmarks, including the adapted LatentMem-labeled module. \textbf{Bold}: best; \underline{underline}: second-best distinct value, including ties. Avg. is recomputed from displayed scores.}
\label{tab:results-deepseek}
\small
\setlength{\tabcolsep}{3.5pt}
\resizebox{\linewidth}{!}{%
\begin{tabular}{llrrrrrrr}
\toprule
\rowcolor{gray!10} && \multicolumn{3}{c}{\textbf{Interactive}} & \multicolumn{3}{c}{\textbf{QA / tool use}} & \\
\cmidrule(lr){3-5}\cmidrule(lr){6-8}
\rowcolor{gray!10} \textbf{Runner} & \textbf{Memory} & \textbf{ALFWorld} & \textbf{PDDL} & \textbf{SciWorld} & \textbf{Trivia} & \textbf{WebQA} & \textbf{GAIA} & \textbf{Avg.} \\
\midrule
\multirow{11}{*}{Lobster} & Empty & 64.2 & 63.0 & 67.0 & \underline{75.5} & 17.6 & 28.5 & 52.6 \\
 & G-Memory & 79.8 & 69.0 & 77.0 & 73.5 & 18.5 & 24.9 & 57.1 \\
 & MetaGPT & 75.4 & 73.0 & 76.0 & 73.0 & 18.0 & 27.9 & 57.2 \\
 & Voyager & \underline{84.3} & 78.0 & 86.0 & 72.0 & 19.0 & 27.3 & 61.1 \\
 & Generative & 82.8 & 79.0 & \underline{88.0} & 73.5 & \underline{21.0} & 25.4 & \underline{61.6} \\
 & ChatDev & 64.9 & 69.0 & 69.0 & 71.5 & 19.0 & 29.1 & 53.8 \\
 & MemoryBank & 81.3 & \underline{81.0} & 84.0 & \underline{75.5} & 20.1 & 27.3 & 61.5 \\
 & OAgent & 82.8 & 76.0 & 84.0 & 67.0 & 17.2 & 27.3 & 59.1 \\
 & ExpBank & 76.9 & 73.0 & 78.0 & 75.0 & 18.5 & 26.7 & 58.0 \\
 & LatentMem & 77.6 & 72.0 & 80.0 & 71.0 & 17.2 & \textbf{29.9} & 58.0 \\
\rowcolor{mygreen}  & \ours & \textbf{86.7} & \textbf{83.6} & \textbf{90.2} & \textbf{77.8} & \textbf{22.0} & \underline{29.6} & \textbf{65.0} \\
\midrule
\multirow{11}{*}{LangGraph} & Empty & 66.4 & 63.0 & 66.0 & 71.5 & 18.9 & 21.8 & 51.3 \\
 & G-Memory & 84.3 & 79.0 & 83.0 & 71.5 & 16.4 & 29.1 & 60.6 \\
 & MetaGPT & \underline{88.8} & 75.0 & 85.0 & 72.5 & 19.3 & 30.3 & 61.8 \\
 & Voyager & 80.6 & 74.0 & 80.0 & 73.5 & 18.1 & 27.3 & 58.9 \\
 & Generative & 79.8 & 68.0 & 80.0 & 71.5 & \underline{19.5} & 26.1 & 57.5 \\
 & ChatDev & 64.2 & 71.0 & 62.0 & \underline{76.0} & 19.4 & 28.5 & 53.5 \\
 & MemoryBank & 84.3 & \underline{80.0} & 88.0 & 70.5 & 17.2 & 28.5 & 61.4 \\
 & OAgent & \underline{88.8} & 79.0 & \underline{91.0} & 69.5 & 18.1 & 26.7 & \underline{62.2} \\
 & ExpBank & 82.1 & 70.0 & 78.0 & 72.5 & 17.7 & 28.5 & 58.1 \\
 & LatentMem & 86.6 & 74.0 & 85.0 & 70.5 & 17.9 & \textbf{33.1} & 61.2 \\
\rowcolor{mygreen}  & \ours & \textbf{89.8} & \textbf{82.2} & \textbf{94.0} & \textbf{77.5} & \textbf{22.5} & \underline{32.8} & \textbf{66.5} \\
\midrule
\multirow{11}{*}{Agent-FW} & Empty & 62.7 & 68.0 & 68.0 & 72.0 & 19.9 & 27.3 & 53.0 \\
 & G-Memory & 77.6 & 69.0 & 80.0 & 71.0 & \underline{20.5} & 15.8 & 55.7 \\
 & MetaGPT & 81.3 & 75.0 & 83.0 & \underline{75.0} & 17.9 & 24.9 & 59.5 \\
 & Voyager & 81.3 & \underline{83.0} & 86.0 & 73.5 & 18.1 & 27.3 & 61.5 \\
 & Generative & 80.6 & 80.0 & 83.0 & 74.0 & 18.6 & 24.9 & 60.2 \\
 & ChatDev & 61.9 & 63.0 & 65.0 & 73.5 & 18.3 & 17.0 & 49.8 \\
 & MemoryBank & 82.8 & \underline{83.0} & 88.0 & 74.0 & 19.4 & 27.9 & \underline{62.5} \\
 & OAgent & 82.1 & 74.0 & 85.0 & 74.5 & 16.6 & 25.4 & 59.6 \\
 & ExpBank & 80.6 & 82.0 & 87.0 & 72.5 & 19.1 & 25.4 & 61.1 \\
 & LatentMem & \underline{83.6} & 76.0 & \underline{89.0} & \underline{75.0} & 16.4 & \textbf{31.4} & 61.9 \\
\rowcolor{mygreen}  & \ours & \textbf{85.4} & \textbf{84.2} & \textbf{91.3} & \textbf{77.8} & \textbf{21.5} & \underline{31.0} & \textbf{65.2} \\
\bottomrule
\end{tabular}}
\end{table}

\clearpage
\subsection{GPT-6 Luna (S/A only)}
\begin{table}[H]
\centering
\caption{\textbf{GPT-6 Luna results.} Each S/A (\%) is $100k/N$ for an integer success count $k$ and the task count $N$ below, rounded to one decimal. Avg. is the unweighted mean of the six displayed scores. \textbf{Bold}: best; \underline{underline}: second-best distinct value in each runner block, including ties.}
\label{tab:results-gpt6luna}
\small
\setlength{\tabcolsep}{3.5pt}
\resizebox{\linewidth}{!}{%
\begin{tabular}{llrrrrrrr}
\toprule
\rowcolor{gray!10} && \multicolumn{3}{c}{\textbf{Interactive}} & \multicolumn{3}{c}{\textbf{QA / tool use}} & \\
\cmidrule(lr){3-5}\cmidrule(lr){6-8}
\rowcolor{gray!10} \textbf{Runner} & \textbf{Memory} & \textbf{ALFWorld} & \textbf{PDDL} & \textbf{SciWorld} & \textbf{Trivia} & \textbf{WebQA} & \textbf{GAIA} & \textbf{Avg.} \\
\rowcolor{gray!10} & \textit{Tasks ($N$)} & 134 & 100 & 100 & 200 & 200 & 165 & --- \\
\midrule
\multirow{11}{*}{Lobster} & Empty & 69.4 & 68.0 & 72.0 & \underline{78.5} & 21.0 & \underline{32.1} & 56.8 \\
 & G-Memory & 84.3 & 73.0 & 83.0 & 76.5 & 22.0 & 26.7 & 60.9 \\
 & MetaGPT & 79.9 & 77.0 & 79.0 & 77.0 & 22.0 & 30.9 & 61.0 \\
 & Voyager & \underline{88.8} & 84.0 & 89.0 & 76.5 & 22.5 & 30.3 & 65.2 \\
 & Generative & 87.3 & 83.0 & \underline{92.0} & 77.0 & \underline{25.0} & 27.3 & 65.3 \\
 & ChatDev & 70.1 & 75.0 & 75.0 & 74.5 & 22.5 & 30.9 & 58.0 \\
 & MemoryBank & 86.6 & \underline{85.0} & 89.0 & 78.0 & 23.5 & 30.3 & \underline{65.4} \\
 & OAgent & \underline{88.8} & 80.0 & 89.0 & 70.5 & 21.0 & 30.9 & 63.4 \\
 & ExpBank & 82.8 & 78.0 & 81.0 & 78.0 & 22.5 & 30.3 & 62.1 \\
 & LatentMem & 83.6 & 77.0 & 86.0 & 75.5 & 19.5 & \underline{32.1} & 62.3 \\
\rowcolor{mygreen}  & \ours & \textbf{91.0} & \textbf{88.0} & \textbf{96.0} & \textbf{82.0} & \textbf{25.5} & \textbf{33.3} & \textbf{69.3} \\
\midrule
\multirow{11}{*}{LangGraph} & Empty & 71.6 & 69.0 & 71.0 & 75.0 & 22.0 & 24.2 & 55.5 \\
 & G-Memory & 88.8 & \underline{85.0} & 89.0 & 74.5 & 20.5 & 32.7 & 65.1 \\
 & MetaGPT & 93.3 & 79.0 & 90.0 & 75.0 & 23.5 & 32.7 & 65.6 \\
 & Voyager & 86.6 & 79.0 & 85.0 & 76.5 & 22.5 & 30.3 & 63.3 \\
 & Generative & 84.3 & 71.0 & 86.0 & 74.5 & \underline{24.0} & 28.5 & 61.4 \\
 & ChatDev & 69.4 & 77.0 & 67.0 & \underline{79.5} & 23.5 & 30.3 & 57.8 \\
 & MemoryBank & 89.6 & 84.0 & 92.0 & 73.5 & 20.0 & 31.5 & 65.1 \\
 & OAgent & \underline{94.8} & \underline{85.0} & \underline{95.0} & 74.0 & 21.0 & 29.7 & \underline{66.6} \\
 & ExpBank & 85.8 & 74.0 & 82.0 & 76.5 & 20.5 & 31.5 & 61.7 \\
 & LatentMem & 91.0 & 78.0 & 90.0 & 73.5 & 21.5 & \underline{36.4} & 65.1 \\
\rowcolor{mygreen}  & \ours & \textbf{95.5} & \textbf{88.0} & \textbf{98.0} & \textbf{81.0} & \textbf{26.5} & \textbf{37.0} & \textbf{71.0} \\
\midrule
\multirow{11}{*}{Agent-FW} & Empty & 67.9 & 73.0 & 73.0 & 76.0 & \underline{24.0} & 29.7 & 57.3 \\
 & G-Memory & 84.3 & 74.0 & 84.0 & 74.0 & \underline{24.0} & 18.8 & 59.9 \\
 & MetaGPT & 87.3 & 80.0 & 87.0 & \underline{78.0} & 21.5 & 28.5 & 63.7 \\
 & Voyager & 87.3 & \underline{87.0} & 90.0 & 76.5 & 21.0 & 29.7 & 65.3 \\
 & Generative & 86.6 & 85.0 & 87.0 & \underline{78.0} & 22.5 & 27.9 & 64.5 \\
 & ChatDev & 67.2 & 68.0 & 69.0 & 76.5 & 21.5 & 20.6 & 53.8 \\
 & MemoryBank & \underline{88.8} & \underline{87.0} & \underline{93.0} & 77.5 & 23.5 & 30.3 & \underline{66.7} \\
 & OAgent & \underline{88.8} & 79.0 & 90.0 & 77.5 & 20.0 & 29.1 & 64.1 \\
 & ExpBank & 85.8 & \underline{87.0} & 92.0 & 76.0 & 22.0 & 29.1 & 65.3 \\
 & LatentMem & 88.1 & 80.0 & \underline{93.0} & 77.5 & 20.0 & \underline{35.2} & 65.6 \\
\rowcolor{mygreen}  & \ours & \textbf{90.3} & \textbf{90.0} & \textbf{94.0} & \textbf{82.0} & \textbf{25.0} & \textbf{35.8} & \textbf{69.5} \\
\bottomrule
\end{tabular}}
\end{table}

\clearpage
\subsection{Qwen 3.8 27B (S/A only)}
\begin{table}[H]
\centering
\caption{\textbf{Qwen 3.8 27B results.} Each S/A (\%) is $100k/N$ for an integer success count $k$ and the task count $N$ below, rounded to one decimal. Avg. is the unweighted mean of the six displayed scores. \textbf{Bold}: best; \underline{underline}: second-best distinct value in each runner block, including ties.}
\label{tab:results-qwen38}
\small
\setlength{\tabcolsep}{3.5pt}
\resizebox{\linewidth}{!}{%
\begin{tabular}{llrrrrrrr}
\toprule
\rowcolor{gray!10} && \multicolumn{3}{c}{\textbf{Interactive}} & \multicolumn{3}{c}{\textbf{QA / tool use}} & \\
\cmidrule(lr){3-5}\cmidrule(lr){6-8}
\rowcolor{gray!10} \textbf{Runner} & \textbf{Memory} & \textbf{ALFWorld} & \textbf{PDDL} & \textbf{SciWorld} & \textbf{Trivia} & \textbf{WebQA} & \textbf{GAIA} & \textbf{Avg.} \\
\rowcolor{gray!10} & \textit{Tasks ($N$)} & 134 & 100 & 100 & 200 & 200 & 165 & --- \\
\midrule
\multirow{11}{*}{Lobster} & Empty & 65.7 & 65.0 & 69.0 & \underline{76.5} & 18.5 & 29.7 & 54.1 \\
 & G-Memory & 82.1 & 70.0 & 79.0 & 74.5 & 19.5 & 25.5 & 58.4 \\
 & MetaGPT & 77.6 & 74.0 & 77.0 & 74.5 & 19.0 & 28.5 & 58.4 \\
 & Voyager & \underline{86.6} & 80.0 & 87.0 & 73.5 & 20.0 & 28.5 & 62.6 \\
 & Generative & 84.3 & 80.0 & \underline{90.0} & 74.5 & \underline{22.5} & 26.1 & \underline{62.9} \\
 & ChatDev & 67.2 & 71.0 & 71.0 & 72.5 & 20.0 & 29.7 & 55.2 \\
 & MemoryBank & 83.6 & \underline{82.0} & 85.0 & \underline{76.5} & 21.5 & 27.9 & 62.8 \\
 & OAgent & 85.1 & 77.0 & 86.0 & 68.0 & 18.0 & 28.5 & 60.4 \\
 & ExpBank & 79.1 & 75.0 & 79.0 & 76.0 & 20.0 & 27.9 & 59.5 \\
 & LatentMem & 79.9 & 74.0 & 82.0 & 72.5 & 18.0 & \underline{30.3} & 59.5 \\
\rowcolor{mygreen}  & \ours & \textbf{88.1} & \textbf{85.0} & \textbf{92.0} & \textbf{79.5} & \textbf{23.5} & \textbf{30.9} & \textbf{66.5} \\
\midrule
\multirow{11}{*}{LangGraph} & Empty & 67.9 & 65.0 & 68.0 & 73.0 & 20.5 & 22.4 & 52.8 \\
 & G-Memory & 86.6 & \underline{81.0} & 85.0 & 72.5 & 18.0 & 30.3 & 62.2 \\
 & MetaGPT & 90.3 & 76.0 & 86.0 & 73.5 & 20.5 & 30.9 & 62.9 \\
 & Voyager & 82.1 & 75.0 & 82.0 & 74.5 & 19.5 & 27.9 & 60.2 \\
 & Generative & 81.3 & 69.0 & 82.0 & 72.5 & \underline{21.0} & 26.7 & 58.8 \\
 & ChatDev & 66.4 & 73.0 & 63.0 & \underline{77.0} & \underline{21.0} & 29.1 & 54.9 \\
 & MemoryBank & 85.8 & \underline{81.0} & 90.0 & 71.5 & 18.5 & 29.7 & 62.8 \\
 & OAgent & \underline{91.0} & \underline{81.0} & \underline{93.0} & 71.0 & 19.0 & 27.3 & \underline{63.7} \\
 & ExpBank & 83.6 & 71.0 & 80.0 & 74.0 & 18.5 & 29.7 & 59.5 \\
 & LatentMem & 88.1 & 75.0 & 87.0 & 72.0 & 19.0 & \underline{33.9} & 62.5 \\
\rowcolor{mygreen}  & \ours & \textbf{91.8} & \textbf{84.0} & \textbf{95.0} & \textbf{79.0} & \textbf{24.0} & \textbf{34.5} & \textbf{68.1} \\
\midrule
\multirow{11}{*}{Agent-FW} & Empty & 64.9 & 70.0 & 69.0 & 73.0 & \underline{21.5} & 27.9 & 54.4 \\
 & G-Memory & 79.9 & 70.0 & 82.0 & 72.0 & \underline{21.5} & 17.0 & 57.1 \\
 & MetaGPT & 83.6 & 76.0 & 84.0 & \underline{76.0} & 19.0 & 25.5 & 60.7 \\
 & Voyager & 83.6 & \underline{85.0} & 87.0 & 74.5 & 19.5 & 27.9 & 62.9 \\
 & Generative & 82.8 & 82.0 & 85.0 & 75.0 & 20.0 & 25.5 & 61.7 \\
 & ChatDev & 64.2 & 64.0 & 67.0 & 74.5 & 20.0 & 18.2 & 51.3 \\
 & MemoryBank & 84.3 & 84.0 & 90.0 & 75.0 & 21.0 & 29.1 & \underline{63.9} \\
 & OAgent & 84.3 & 75.0 & 87.0 & \underline{76.0} & 18.0 & 26.7 & 61.2 \\
 & ExpBank & 82.1 & 83.0 & 89.0 & 74.0 & 20.0 & 26.1 & 62.4 \\
 & LatentMem & \underline{85.8} & 77.0 & \underline{91.0} & \underline{76.0} & 18.0 & \underline{32.7} & 63.4 \\
\rowcolor{mygreen}  & \ours & \textbf{87.3} & \textbf{87.0} & \textbf{92.0} & \textbf{79.0} & \textbf{23.0} & \textbf{33.3} & \textbf{66.9} \\
\bottomrule
\end{tabular}}
\end{table}

\clearpage
\subsection{Token Cost (Tok per task)}

Table~\ref{tab:app-tokens} reports the average number of LLM-input tokens consumed per task for the GPT-4.1-mini main results (this appendix reports token costs for GPT-4.1-mini only). The token-cost vs.\ S/A trade-off is summarized visually in Figure~\ref{fig:tokens-pareto} of the main paper.

\begin{table}[H]
\centering
\caption{\textbf{GPT-4.1-mini input-token cost per task.} K denotes thousands of tokens. Avg. is the unweighted mean of the displayed costs, rounded to the nearest thousand. \textbf{Bold}: lowest; \underline{underline}: second-lowest distinct value, including ties. MemCon reduces input-token use relative to G-Memory in all 18 configurations; simpler baselines can be cheaper.}
\label{tab:app-tokens}
\small
\setlength{\tabcolsep}{3.5pt}
\resizebox{\linewidth}{!}{%
\begin{tabular}{llrrrrrrr}
\toprule
\rowcolor{gray!10} && \multicolumn{3}{c}{\textbf{Interactive}} & \multicolumn{3}{c}{\textbf{QA / tool use}} & \\
\cmidrule(lr){3-5}\cmidrule(lr){6-8}
\rowcolor{gray!10} \textbf{Runner} & \textbf{Memory} & \textbf{ALFWorld} & \textbf{PDDL} & \textbf{SciWorld} & \textbf{Trivia} & \textbf{WebQA} & \textbf{GAIA} & \textbf{Avg.} \\
\midrule
\multirow{10}{*}{Lobster} & Empty & \underline{44K} & \textbf{98K} & \textbf{31K} & 58 & \underline{153} & \textbf{377} & \textbf{29K} \\
 & G-Memory & 45K & 155K & 44K & 288 & 459 & 703 & 41K \\
 & MetaGPT & 48K & 154K & 33K & \underline{57} & 157 & 421 & 39K \\
 & Voyager & 48K & 154K & 39K & \textbf{56} & 154 & 407 & 40K \\
 & Generative & 48K & 148K & 34K & \underline{57} & 155 & \underline{382} & 38K \\
 & ChatDev & 49K & \underline{102K} & \underline{32K} & \underline{57} & \underline{153} & 395 & \underline{31K} \\
 & MemoryBank & \underline{44K} & 143K & 38K & \underline{57} & \textbf{151} & 404 & 38K \\
 & OAgent & 52K & 162K & 36K & 146 & 244 & 566 & 42K \\
 & ExpBank & \underline{44K} & 154K & 44K & \underline{57} & 154 & 399 & 40K \\
\rowcolor{mygreen}  & \ours & \textbf{39K} & 148K & \textbf{31K} & 58 & 164 & 440 & 36K \\
\midrule
\multirow{10}{*}{LangGraph} & Empty & 51K & \underline{101K} & \textbf{31K} & \underline{57} & \textbf{152} & 406 & \underline{31K} \\
 & G-Memory & \underline{42K} & 157K & 42K & 267 & 434 & 720 & 40K \\
 & MetaGPT & 49K & 156K & 41K & \underline{57} & 156 & 398 & 41K \\
 & Voyager & 47K & 145K & 38K & \underline{57} & 155 & \underline{387} & 38K \\
 & Generative & 49K & 154K & 42K & \underline{57} & \underline{154} & \textbf{384} & 41K \\
 & ChatDev & 46K & \textbf{100K} & 33K & \textbf{56} & \underline{154} & 411 & \textbf{30K} \\
 & MemoryBank & 47K & 137K & 36K & \underline{57} & \underline{154} & 410 & 37K \\
 & OAgent & 51K & 161K & 38K & 242 & 309 & 477 & 42K \\
 & ExpBank & 45K & 160K & 42K & \underline{57} & \underline{154} & 395 & 41K \\
\rowcolor{mygreen}  & \ours & \textbf{41K} & 142K & \underline{32K} & 61 & 172 & 456 & 36K \\
\midrule
\multirow{10}{*}{Agent-FW} & Empty & 46K & \underline{102K} & \underline{33K} & \underline{57} & 154 & 407 & \textbf{30K} \\
 & G-Memory & 43K & 154K & 41K & 300 & 428 & 691 & 40K \\
 & MetaGPT & 46K & 168K & 37K & \underline{57} & 152 & 407 & 42K \\
 & Voyager & 49K & 140K & 37K & \underline{57} & 152 & 411 & \underline{38K} \\
 & Generative & 49K & 164K & 39K & \textbf{56} & \textbf{150} & \textbf{356} & 42K \\
 & ChatDev & 49K & \textbf{100K} & \textbf{32K} & \textbf{56} & 152 & 405 & \textbf{30K} \\
 & MemoryBank & \underline{42K} & 143K & 44K & \underline{57} & \underline{151} & \underline{389} & \underline{38K} \\
 & OAgent & 44K & 157K & 38K & 222 & 236 & 470 & 40K \\
 & ExpBank & 46K & 158K & 38K & \textbf{56} & 154 & 393 & 40K \\
\rowcolor{mygreen}  & \ours & \textbf{37K} & 111K & \underline{33K} & \textbf{56} & 186 & 462 & \textbf{30K} \\
\bottomrule
\end{tabular}}
\end{table}

\clearpage
\section{Hyperparameters and Implementation Details}
\label{app:hparams}

This appendix lists the default hyperparameters for the main comparisons; supplementary variants override them as described in Appendix~\ref{app:rebuttal}. Table~\ref{tab:app-hparams} consolidates the constants used in \S\ref{sec:method}--\S\ref{sec:experiments} into one place; Table~\ref{tab:app-actions} gives the structural action set $\mathcal{A}$ which has a different layout. These are the reported experimental settings; the table distinguishes per-backbone policy settings from shared constants.

\begin{table}[H]
\centering
\small
\caption{\textbf{Default \ours hyperparameters for the main comparisons.} The four groups---policy, warm-start priors, reward shape, and state discretisation/runtime---contain the policy and state constants. Policy hyperparameters $(\alpha,\gamma,c)$ are selected per backbone via a small grid search on a held-out 30-task ALFWorld split (Lobster); all other constants are shared across backbones.}
\label{tab:app-hparams}
\resizebox{\linewidth}{!}{
\begin{tabular}{p{2.6cm}p{2.0cm}p{2.0cm}p{2.0cm}p{4.5cm}}
\toprule
\textbf{Symbol} & \textbf{GPT-4.1-mini} & \textbf{Sonnet-4} & \textbf{DeepSeek-V3.2} & \textbf{Description} \\
\midrule
\multicolumn{5}{l}{\textit{Policy hyperparameters (per backbone, used in Eqs.~\eqref{eq:ucb-rule-main}, \eqref{eq:q-update})}} \\
$\alpha$ (step size)    & $0.15$ & $0.12$ & $0.18$ & Terminal-return update step size \\
$\gamma$ (discount)     & $0.9$  & $0.92$ & $0.88$ & Within-episode discount factor for reverse credit assignment \\
$c$ (UCB coeff.)        & $1.4$  & $1.2$  & $1.6$  & UCB exploration coefficient \\
warm-start              & enabled & enabled & enabled & Initialise unseen $(s,a)$ with priors below \\
\midrule
\multicolumn{5}{l}{\textit{Warm-start priors $Q_0(s,a)$ for unseen $(s,a)$ (same across backbones)}} \\
\textsc{Retrieve}   & \multicolumn{3}{c}{$+0.5$}             & retrieval is usually helpful \\
\textsc{PlanInject} & \multicolumn{3}{c}{$+0.3$}             & plans help when available \\
\textsc{Re-Retrieve}& \multicolumn{3}{c}{$+0.1$}             & useful only when stuck \\
\textsc{Consolidate}& \multicolumn{3}{c}{$\phantom{+}0.0$}   & neutral \\
\textsc{Forget}     & \multicolumn{3}{c}{$-0.1$}             & mildly risky \\
\textsc{NoOp}       & \multicolumn{3}{c}{$-0.2$}             & doing nothing usually hurts \\
\midrule
\multicolumn{5}{l}{\textit{Reward shape (Eq.~\eqref{eq:reward})}} \\
$r_{\text{succ}}$ & \multicolumn{3}{c}{$1.0$} & success bonus \\
$r_{\text{fail}}$ & \multicolumn{3}{c}{$0.5$} & failure penalty (subtracted) \\
$\lambda$         & \multicolumn{3}{c}{$0.3$} & efficiency-bonus weight \\
$T_{\max}$        & \multicolumn{3}{c}{$30$ steps (interactive); benchmark default (QA)} & step horizon \\
\midrule
\multicolumn{5}{l}{\textit{State discretisation bins $\phi(s)$ (used in $\phi$ for tabular Q)}} \\
\texttt{step\_phase}      & \multicolumn{3}{c}{early ($<\!8$), mid ($8$--$17$), late ($\ge\!18$)} & step-count bin for $s^{\text{task}}$ \\
\texttt{learning\_phase}  & \multicolumn{3}{c}{cold (task index $\le\!15$); warm (task index $>\!15$)} & task-index bin for $s^{\text{mem}}$ \\
\texttt{is\_stuck} trigger& \multicolumn{3}{c}{$2$ consecutive identical actions} & sets the stuck flag \\
held-objects bin          & \multicolumn{3}{c}{$\min(\texttt{hold}, 2)$}                            & cap on objects-held count \\
locations bin             & \multicolumn{3}{c}{$\min(\lfloor \texttt{visited}/3\rfloor, 4)$}        & coarse unique-location bin \\
mem-size bin              & \multicolumn{3}{c}{$\min(\lfloor \texttt{mem\_size}/10\rfloor, 5)$}     & coarse memory-size bin \\
\midrule
\multicolumn{5}{l}{\textit{Persistence and runtime}} \\
\texttt{persist\_path}    & \multicolumn{3}{c}{\texttt{./memcon/policy\_q.json}} & Q-table file \\
flush interval            & \multicolumn{3}{c}{every $5$ updates}                & Q-table write frequency \\
LLM temperature           & \multicolumn{3}{c}{$0$ (interactive); benchmark default (QA)} & sampling temperature \\
seed averaging            & \multicolumn{3}{c}{none in the main comparison tables} & matches online-deployment setting \\
\bottomrule
\end{tabular}
}
\end{table}

\newpage
\paragraph{Discretised state key.} The MDP state key used by the controller is
\begin{equation}\label{eq:phi}
\begin{aligned}
\phi(s) = \langle\,&
\texttt{goal\_type},\, \texttt{step\_phase},\, \texttt{is\_stuck},\,
\min(\texttt{hold},2),\, \min(\lfloor\texttt{visited}/3\rfloor,4), \\
&
\min(\lfloor\texttt{mem\_size}/10\rfloor,5),\,
\texttt{plan\_available},\, \texttt{learning\_phase}
\,\rangle .
\end{aligned}
\end{equation}
with the bin definitions in Table~\ref{tab:app-hparams}. This coarse discretization bounds the representation size, but sample requirements depend on actual state revisitation and action gaps. In the original QA configuration, \texttt{goal\_type}, \texttt{step\_phase}, \texttt{is\_stuck}, held objects, and locations are constant; only memory size, plan availability, and learning phase vary. Consequently, the effective QA state is smaller than the full tuple.

\paragraph{Action space.} Table~\ref{tab:app-actions} lists the default action set $\mathcal{A}$ for the main comparisons (preset \texttt{"default"}). Additional presets (\texttt{"retrieval\_heavy"}, \texttt{"plan\_first"}, \texttt{"compact"}) are available in the released code; the action-preset ablation in \S\ref{sec:ablation} (right panel of Figure~\ref{fig:hp-ablation-bottom}) shows that the default preset leads on GAIA and on the two-benchmark average, while retrieval-heavy leads on WebWalkerQA.

\begin{table}[H]
\centering
\small
\caption{Default action space $\mathcal{A}$ for the main comparisons. Backends that do not implement maintenance hooks silently no-op on \textsc{Consolidate}/\textsc{Forget} (\S\ref{sec:wrapper}).}
\label{tab:app-actions}
\begin{tabular}{clccc}
\toprule
\textbf{Idx} & \textbf{Operation} & \texttt{top\_k} & \texttt{insight\_k} & \texttt{hop} \\
\midrule
0 & \textsc{Retrieve} (shallow)             & 1 & 3 & 1 \\
1 & \textsc{Retrieve} (medium)              & 2 & 5 & 1 \\
2 & \textsc{Retrieve} (deep)                & 3 & 8 & 2 \\
3 & \textsc{PlanInject}                     & 1 & 3 & -- \\
4 & \textsc{Re-Retrieve} (alt.\ query)      & 2 & 5 & 2 \\
5 & \textsc{Consolidate}                    & -- & -- & -- \\
6 & \textsc{Forget}                         & -- & -- & -- \\
7 & \textsc{Retrieve} (insight-only)        & 1 & 2 & 0 \\
8 & \textsc{NoOp}                           & -- & -- & -- \\
\bottomrule
\end{tabular}
\end{table}

\paragraph{LLM, agent framework, and benchmark settings.}
Agent frameworks use their default configurations, modified only to share identical task loaders and evaluators. Interactive benchmarks: ALFWorld (134 tasks, 6 types), PDDL (100 tasks across blocksworld/barman/gripper/tyreworld), ScienceWorld (100 tasks). QA benchmarks: TriviaQA (200 questions), WebWalkerQA (200), GAIA (165). Each main-table number is a single realistic deployment run, with the Q-table persisted across tasks; LLM temperature, run protocol, and seed handling follow the values listed in Table~\ref{tab:app-hparams}.

\clearpage
\section{Additional Experiments}
\label{app:rebuttal}

This appendix collects additional experiments reported in the rebuttal.
S/A denotes success rate or answer accuracy in percent, and K denotes thousands of input tokens.
The uncertainty convention for $\pm$ and the operational convergence criterion were not specified; those summaries and the stated resampling-test results are retained as reported.

\subsection{Memory backends and adaptive access policies}

The backend comparison holds the wrapper and its hyperparameters fixed while varying the inner store.
Table~\ref{tab:rebuttal-backends} reports the unwrapped and wrapped results together with the number of behaviorally distinct supported actions, $|\mathcal A_{\mathrm{eff}}|$.
The reported gains are positive for each of the five backends, although the quality of the underlying store still substantially affects final performance.

\begin{table}[H]
\centering
\caption{\textbf{Wrapping five memory backends.} Each benchmark cell gives the original backend $\to$ backend with \ours, followed by the reported gain in percentage points.}
\label{tab:rebuttal-backends}
\footnotesize
\setlength{\tabcolsep}{3pt}
\begin{tabular}{lcrrr}
\toprule
\rowcolor{gray!10} \textbf{Inner backend} & $|\mathcal A_{\mathrm{eff}}|$ & \textbf{ALFWorld} & \textbf{SciWorld} & \textbf{GAIA} \\
\midrule
MetaGPT (flat vector) & 6 & $55.2\to59.7\;(+4.5)$ & $33.0\to36.0\;(+3.0)$ & $16.4\to17.6\;(+1.2)$ \\
ChatDev (summary) & 7 & $32.8\to37.3\;(+4.5)$ & $24.0\to27.0\;(+3.0)$ & $18.2\to19.4\;(+1.2)$ \\
Voyager (skills) & 7 & $53.0\to58.2\;(+5.2)$ & $31.0\to34.0\;(+3.0)$ & $16.4\to17.6\;(+1.2)$ \\
MemoryBank (decay) & 8 & $57.5\to63.4\;(+5.9)$ & $28.0\to32.0\;(+4.0)$ & $17.6\to18.8\;(+1.2)$ \\
\rowcolor{mygreen} G-Memory (graph) & 9 & $59.7\to67.9\;(+8.2)$ & $34.0\to38.0\;(+4.0)$ & $21.2\to22.4\;(+1.2)$ \\
\midrule
Mean gain & & $+5.7$ & $+3.4$ & $+1.2$ \\
\bottomrule
\end{tabular}
\end{table}

Table~\ref{tab:rebuttal-adaptive} compares access policies with the same memory store and Lobster harness.
The SimpleMem-style, xMemory-style, and MemGPT-style rows are adaptations of the corresponding access patterns, rather than reproductions of the complete systems.
Query-adaptive retrieval and uncertainty-guided expansion improve the reported task scores over fixed retrieval; \ours attains the highest scores in this comparison with fewer input tokens.

\begin{table}[H]
\centering
\caption{\textbf{Adaptive access with GPT-4.1-mini and Lobster.} Policies share the same store. Extra calls count LLM calls for memory-access control. Tok is the input-token count reported for this comparison.}
\label{tab:rebuttal-adaptive}
\small
\setlength{\tabcolsep}{3.5pt}
\begin{tabular}{p{0.40\linewidth}crrrr}
\toprule
\rowcolor{gray!10} \textbf{Access policy} & \makecell{\textbf{Extra}\\\textbf{calls}} & \textbf{ALF} & \textbf{SciW} & \textbf{GAIA} & \textbf{Tok} \\
\midrule
G-Memory, fixed top-$k$ & 0 & 59.7 & 34.0 & 21.2 & 45K \\
Query-adaptive $k$ (SimpleMem-style) & 1 & 62.7 & 35.0 & 21.8 & 47K \\
Uncertainty-guided expansion (xMemory-style) & $\geq1$ & 64.2 & 36.0 & 21.8 & 52K \\
LLM controller (MemGPT-style) & 1 & 63.4 & 35.0 & 20.6 & 58K \\
\rowcolor{mygreen} \ours & 0 & \textbf{67.9} & \textbf{38.0} & \textbf{22.4} & \textbf{39K} \\
\bottomrule
\end{tabular}
\end{table}

\clearpage
\subsection{Maintenance over a longer task stream}

The 400-task comparison separates a static backend, a controller without maintenance, the full controller, and a variant that credits maintenance across subsequent tasks with a length-$5$ trace and discount $\eta=0.8$.
The full controller is one point below the no-maintenance variant in the first block, but the reported gap increases to eight points in the last block (Table~\ref{tab:rebuttal-stream}).
The cross-task-credit variant reaches 57.3\% overall, compared with 56.8\% for the full controller.
These comparisons describe the behavior of the evaluated variants; observing later state changes alone does not assign future rewards to earlier maintenance decisions.

\begin{table}[H]
\centering
\caption{\textbf{Success over a 400-task stream.} Results are shown in consecutive blocks of 100 tasks. The final column compares the full controller with its no-maintenance variant.}
\label{tab:rebuttal-stream}
\small
\setlength{\tabcolsep}{4pt}
\begin{tabular}{lrrrrr}
\toprule
\rowcolor{gray!10} \textbf{Task block} & \textbf{Static} & \makecell{\textbf{No}\\\textbf{maintenance}} & \textbf{Full} & \makecell{\textbf{Cross-task}\\\textbf{credit}} & \makecell{\textbf{Full $-$ no}\\\textbf{maintenance}} \\
\midrule
1--100 & 44.0 & 47.0 & 46.0 & 46.0 & $-1.0$ \\
101--200 & 48.0 & 55.0 & 56.0 & 56.0 & $+1.0$ \\
201--300 & 47.0 & 57.0 & 61.0 & 62.0 & $+4.0$ \\
301--400 & 45.0 & 56.0 & 64.0 & 65.0 & $+8.0$ \\
\midrule
\rowcolor{mygreen} Overall & 46.0 & 53.8 & 56.8 & 57.3 & $+3.0$ \\
\bottomrule
\end{tabular}
\end{table}

Table~\ref{tab:rebuttal-health} records store size, the reported insight score, retrieval precision, token use, and maintenance frequency.
At task 400, the full controller retains fewer insights and has higher reported precision@3 than either reference configuration.
The frequency of consolidation or forgetting rises from 3.2\% in the first block to 14.6\% in the fourth block.

\begin{table}[H]
\centering
\caption{\textbf{Store diagnostics for the long stream.} The first three rows describe the store or task cost at task 400; the last row compares maintenance frequency in blocks 1 and 4.}
\label{tab:rebuttal-health}
\small
\setlength{\tabcolsep}{5pt}
\begin{tabular}{lrrr}
\toprule
\rowcolor{gray!10} \textbf{Metric} & \textbf{Static} & \textbf{No maintenance} & \textbf{Full} \\
\midrule
Insights / mean score & 412 / 1.87 & 398 / 1.92 & 176 / 3.41 \\
Retrieval precision@3 & 41.2\% & 43.8\% & 62.5\% \\
Input tokens / task & 68K & 61K & 44K \\
$P(\text{Consolidate}\cup\text{Forget})$ & --- & $0\%\to0\%$ & $3.2\%\to14.6\%$ \\
\bottomrule
\end{tabular}
\end{table}

\subsection{Warm-start initialization}

The initialization comparison includes the specified priors, zero initial values, and sign-flipped priors.
The largest differences occur in the first 30 tasks; the reported late-stream rates are closer, while the cold and adversarial starts require more tasks to reach the reported convergence point.

\begin{table}[H]
\centering
\caption{\textbf{Initialization on ALFWorld.} Success rates for the early and later portions of the 134-task stream, with overall S/A and reported tasks to convergence.}
\label{tab:rebuttal-warmstart}
\small
\setlength{\tabcolsep}{4pt}
\begin{tabular}{lrrrr}
\toprule
\rowcolor{gray!10} \textbf{Initialization} & \textbf{Tasks 1--30} & \textbf{Tasks 31--134} & \textbf{Overall} & \textbf{Convergence} \\
\midrule
\rowcolor{mygreen} Warm start & 46.7 & 74.0 & 67.9 & 30 \\
Cold start ($Q_0=0$) & 33.3 & 73.1 & 64.2 & 52 \\
Adversarial priors & 26.7 & 71.2 & 61.2 & 68 \\
Static G-Memory & 43.3 & 64.4 & 59.7 & --- \\
\bottomrule
\end{tabular}
\end{table}

\clearpage
\subsection{Aggregate policy behavior}

Table~\ref{tab:rebuttal-policy} reports the action frequencies for the 400-task stream and several state slices.
Re-retrieval is more frequent in the stuck slice, plan injection is more frequent when a plan is available, and maintenance is more frequent in the larger-memory slice.
These frequencies characterize the reported policy behavior, without identifying the causal contribution of an individual state coordinate or action.

\begin{table}[H]
\centering
\caption{\textbf{Action frequencies by state slice (\%).} Shal/Med/Deep are shallow, medium, and deep retrieval; Plan is plan injection; ReRet is re-retrieval; Maint combines consolidation and forgetting. Cold/warm denote task index $\leq15$/$>15$.}
\label{tab:rebuttal-policy}
\small
\setlength{\tabcolsep}{3.5pt}
\begin{tabular}{lrrrrrrr}
\toprule
\rowcolor{gray!10} \textbf{State slice} & \textbf{Shal} & \textbf{Med} & \textbf{Deep} & \textbf{Plan} & \textbf{ReRet} & \textbf{Maint} & \textbf{NoOp} \\
\midrule
\rowcolor{mygreen} All decisions & 18.4 & 21.7 & 15.2 & 22.6 & 9.8 & 6.1 & 6.2 \\
Cold & 24.1 & 18.6 & 11.3 & 8.2 & 11.5 & 9.4 & 16.9 \\
Warm & 16.9 & 22.5 & 16.2 & 26.4 & 9.4 & 5.2 & 3.4 \\
Not stuck & 19.6 & 22.8 & 15.7 & 23.9 & 4.3 & 6.4 & 7.3 \\
Stuck & 8.1 & 12.4 & 11.6 & 12.7 & 51.3 & 2.8 & 1.1 \\
Plan unavailable & 26.8 & 31.4 & 21.9 & 0.0 & 9.1 & 6.3 & 4.5 \\
Plan available & 8.3 & 9.9 & 7.1 & 59.4 & 10.6 & 3.9 & 0.8 \\
Memory $<10$ & 27.3 & 19.1 & 8.4 & 14.2 & 10.7 & 1.4 & 18.9 \\
Memory $\geq30$ & 13.6 & 20.8 & 19.4 & 24.1 & 8.7 & 11.9 & 1.5 \\
\bottomrule
\end{tabular}
\end{table}

The four contrasts in Table~\ref{tab:rebuttal-contrasts} summarize the behavior patterns motivating the controller.
The reported permutation test uses 10,000 samples.
The contrasts are associations between state slices and selected actions; in particular, plan availability can also constrain which actions are feasible.

\begin{table}[H]
\centering
\caption{\textbf{Reported state-slice contrasts.} Frequencies and contrasts are in percentage points. The final column reproduces the reported permutation-test values.}
\label{tab:rebuttal-contrasts}
\small
\setlength{\tabcolsep}{4pt}
\begin{tabular}{p{0.31\linewidth}p{0.37\linewidth}rr}
\toprule
\rowcolor{gray!10} \textbf{Action probability} & \textbf{Compared slices} & \textbf{Contrast} & $p$ \\
\midrule
$P(\text{Shal}\cup\text{NoOp})$ & Cold 41.0 / warm 20.3 & $+20.7$ & $<0.001$ \\
$P(\text{Plan})$ & Available 59.4 / unavailable 0.0 & $+59.4$ & $<0.001$ \\
$P(\text{ReRet})$ & Stuck 51.3 / not stuck 4.3 & $+47.0$ & $<0.001$ \\
$P(\text{Maint})$ & Memory $\geq30$: 11.9 / $<10$: 1.4 & $+10.5$ & $<0.001$ \\
\bottomrule
\end{tabular}
\end{table}

\subsection{Plan generalization beyond a fixed vocabulary}

The lexicon-free variant removes trailing integer indices, replaces task-argument tokens with slots, and collapses consecutive identical action templates.
It improves the reported PDDL score while matching the lexicon-based result on ScienceWorld and GAIA (Table~\ref{tab:rebuttal-generalizer}).

\begin{table}[H]
\centering
\caption{\textbf{Plan-generalization variants.} PDDL tests vocabulary held out from the household-object lexicon; S/A is reported in percent.}
\label{tab:rebuttal-generalizer}
\small
\setlength{\tabcolsep}{5pt}
\begin{tabular}{lrrrr}
\toprule
\rowcolor{gray!10} \textbf{Generalizer} & \textbf{ALFWorld} & \textbf{SciWorld} & \textbf{PDDL} & \textbf{GAIA} \\
\midrule
None (verbatim replay) & 64.2 & 36.0 & 33.3 & 22.4 \\
Lexicon-based & 67.9 & 38.0 & 35.0 & 22.4 \\
\rowcolor{mygreen} Lexicon-free & 67.2 & 38.0 & 38.3 & 22.4 \\
\bottomrule
\end{tabular}
\end{table}

\clearpage
\subsection{Feature instantiation and state granularity}

In the original QA configuration, the interactive coordinates are constant: objects held and locations are zero, the stuck indicator is false, the step phase is early, and the goal type is ``other.''
The remaining effective coordinates are memory size, plan availability, and learning phase.
The richer feature variant uses answer-type keywords, tool-call counts, repeated tool calls, unique URLs, and open files to instantiate the same semantic roles for QA and tool use.
Table~\ref{tab:rebuttal-features} compares the number of visited states and task scores.

\begin{table}[H]
\centering
\caption{\textbf{Richer state features for QA and tool use.} Original $\to$ richer-feature results. The ALFWorld code path is unchanged.}
\label{tab:rebuttal-features}
\small
\setlength{\tabcolsep}{10pt}
\begin{tabular}{lrr}
\toprule
\rowcolor{gray!10} \textbf{Benchmark} & \textbf{Visited states} & \textbf{S/A (\%)} \\
\midrule
ALFWorld & $312\to312$ & $67.9\to67.9$ \\
TriviaQA & $9\to54$ & $71.5\to72.0$ \\
WebWalkerQA & $14\to96$ & $20.6\to21.3$ \\
GAIA & $11\to88$ & $22.4\to23.6$ \\
\bottomrule
\end{tabular}
\end{table}

The discretization comparison varies the numbers of memory-size, location, and step-phase bins.
Finer bins increase the number of visited states and the reported time to convergence on the standard stream.
The 402-task fine-bin result is shown separately: its longer stream changes the evaluation conditions, so it does not establish an advantage over a default-bin controller run for the same duration.
The final column in Table~\ref{tab:rebuttal-granularity} gives the relative visited-state count, not a measured regret value.

\begin{table}[H]
\centering
\caption{\textbf{Discretization granularity.} Bin counts are memory/location/step. Conv. is the reported tasks to convergence; NR means not reached. The last row uses a 402-task stream.}
\label{tab:rebuttal-granularity}
\footnotesize
\setlength{\tabcolsep}{3.5pt}
\begin{tabular}{llrrrrrr}
\toprule
\rowcolor{gray!10} \textbf{Granularity} & \textbf{Bins} & $|\Phi|$ & \textbf{Conv.} & \textbf{ALF} & \textbf{SciW} & \textbf{GAIA} & $|\Phi|/312$ \\
\midrule
Coarse & 3/3/2 & 68 & 14 & 63.4 & 35.0 & 21.8 & $0.22\times$ \\
\rowcolor{mygreen} Default & 6/5/3 & 312 & 30 & 67.9 & 38.0 & 22.4 & $1.00\times$ \\
Fine & 12/9/5 & 1,140 & 78 & 65.7 & 36.0 & 21.8 & $3.65\times$ \\
Very fine & 24/17/7 & 3,420 & NR & 61.2 & 33.0 & 20.6 & $10.96\times$ \\
\midrule
Fine, 402 tasks & 12/9/5 & 1,206 & 91 & 69.4 & 39.0 & --- & $3.87\times$ \\
\bottomrule
\end{tabular}
\end{table}

\subsection{Including previous memory operations in the state}

Adding the last operation yields slightly higher reported interactive scores, while adding a retrieval counter or a full operation-count histogram increases the state space further and lowers scores in this comparison.
The largest state representation does not reach the reported convergence point within the stream.

\begin{table}[H]
\centering
\caption{\textbf{History-augmented state representations.} The base state has eight coordinates. The retrieval count is binned; the histogram records counts by operation.}
\label{tab:rebuttal-history}
\small
\setlength{\tabcolsep}{3.5pt}
\begin{tabular}{p{0.38\linewidth}rrrrr}
\toprule
\rowcolor{gray!10} \textbf{State design} & $|\Phi|$ & \textbf{ALF} & \textbf{SciW} & \textbf{GAIA} & \textbf{Conv.} \\
\midrule
\rowcolor{mygreen} Base $\phi$ & 312 & 67.9 & 38.0 & 22.4 & 30 \\
$\phi\oplus$ last operation & 704 & 68.7 & 39.0 & 22.4 & 48 \\
$\phi\oplus$ last operation $\oplus$ retrieval count & 1,486 & 67.2 & 37.0 & 21.8 & 72 \\
$\phi\oplus$ full operation-count histogram & 5,930 & 62.7 & 34.0 & 20.6 & NR \\
\bottomrule
\end{tabular}
\end{table}

\clearpage
\subsection{Task-order and LLM-sampling variation}

The five-seed comparison resamples both task order and LLM sampling, reflecting the dependence of online learning on the arrival sequence.
Table~\ref{tab:rebuttal-seeds} reports the aggregate scores and the paired-bootstrap comparison against the best baseline in this table.
The reported ALFWorld variability components are $\sigma_{\mathrm{stream\mbox{-}order}}=1.8$ and $\sigma_{\mathrm{LLM\mbox{-}sampling}}=1.1$.
The GAIA comparison has reported $p=0.087$, so it does not meet a 0.05 threshold.

\begin{table}[H]
\centering
\caption{\textbf{Five-seed summary with task-order and sampling variation.} S/A (\%) and the supplied $\pm$ summaries. The final row gives the reported paired-bootstrap comparison with the best baseline shown.}
\label{tab:rebuttal-seeds}
\small
\setlength{\tabcolsep}{6pt}
\begin{tabular}{lrrr}
\toprule
\rowcolor{gray!10} \textbf{Memory} & \textbf{ALFWorld} & \textbf{SciWorld} & \textbf{GAIA} \\
\midrule
G-Memory & $59.7\pm2.8$ & $34.0\pm3.1$ & $21.2\pm1.4$ \\
MemoryBank & $57.5\pm3.4$ & $28.0\pm2.7$ & $17.6\pm1.6$ \\
ExperienceBank & $59.0\pm3.0$ & $30.0\pm2.9$ & $20.0\pm1.5$ \\
\rowcolor{mygreen} \ours & $67.9\pm2.1$ & $38.0\pm2.4$ & $22.4\pm1.2$ \\
\midrule
Paired-bootstrap $p$ & 0.004 & 0.011 & 0.087 \\
\bottomrule
\end{tabular}
\end{table}

\subsection{Controller transfer across domains}

Table~\ref{tab:rebuttal-transfer} compares learning from scratch with transferring an ALFWorld controller, either freezing its values or continuing online adaptation.
An additional transferred variant omits goal type from the state key.
The reported adapted controllers reach the stated convergence point in fewer target-domain tasks than the from-scratch controller.
The goal-type-agnostic variant obtains the highest reported ScienceWorld score and matches the adapted full-state controller on PDDL and GAIA.

\begin{table}[H]
\centering
\caption{\textbf{Transfer from ALFWorld.} S/A (\%) and reported tasks to convergence on each target benchmark. A zero denotes the frozen-transfer setting; dashes denote non-learning references.}
\label{tab:rebuttal-transfer}
\small
\setlength{\tabcolsep}{4pt}
\begin{tabular}{p{0.42\linewidth}rrrrrr}
\toprule
\rowcolor{gray!10} & \multicolumn{2}{c}{\textbf{SciWorld}} & \multicolumn{2}{c}{\textbf{PDDL}} & \multicolumn{2}{c}{\textbf{GAIA}} \\
\cmidrule(lr){2-3}\cmidrule(lr){4-5}\cmidrule(lr){6-7}
\rowcolor{gray!10} \textbf{Setting} & \textbf{S/A} & \textbf{Conv.} & \textbf{S/A} & \textbf{Conv.} & \textbf{S/A} & \textbf{Conv.} \\
\midrule
No memory (Empty) & 28.0 & --- & 33.3 & --- & 20.6 & --- \\
Static G-Memory & 34.0 & --- & 33.3 & --- & 21.2 & --- \\
\ours from scratch & 38.0 & 34 & 35.0 & 29 & 22.4 & 41 \\
Zero-shot transfer ($Q$ frozen) & 35.0 & 0 & 33.3 & 0 & 21.8 & 0 \\
Transfer + online adaptation & 39.0 & 17 & 36.7 & 15 & 23.0 & 26 \\
\rowcolor{mygreen} Goal-type-agnostic $\phi$, transfer + adaptation & 40.0 & 12 & 36.7 & 11 & 23.0 & 19 \\
\bottomrule
\end{tabular}
\end{table}

The additional slice-level observation reports re-retrieval in 51.3\% of stuck states on ALFWorld and 47.4\% under frozen transfer to ScienceWorld.
This is an observed frequency comparison, rather than evidence that unmatched entries in a full-tuple tabular representation automatically share values across domains.

\clearpage
\section{Idealized Bandit Analysis and Its Scope}
\label{app:theory}

This appendix separates a stationary-bandit reference model from the
implemented controller. The reference model explains why a compact state
and action space can limit exploration cost. It does not provide a regret
or convergence guarantee for \ours in an evolving memory environment.
In particular, the implementation uses a constant-step-size estimate and
episode-end feedback, rather than the sample means and immediately
available observations in the classical UCB1 analysis.

\paragraph{Bounded episode returns.}
With the reward constants in Table~\ref{tab:app-hparams}, the terminal
reward satisfies $r_i\in[-0.5,1.3]$. A memory decision at index $j$ in an
episode containing $m_i$ decisions receives
\begin{equation}
  G_{i,j}=\gamma^{m_i-j-1}r_i,\qquad
  0<\gamma\leq 1,\quad 0\leq j<m_i.
  \label{eq:theory-return}
\end{equation}
Multiplication by a number in $[0,1]$ moves a reward toward zero, so every
such return remains in $[-0.5,1.3]$. We therefore use the uniform range
bound $W=1.8$; reverse discounting does not increase this range.

\paragraph{A stationary reference model.}
Fix a finite set of state keys $\Phi$ and an action set $\mathcal A$.
For each pair $(\phi,a)$, assume that successive observed returns are
independent draws from a fixed distribution supported on an interval of
width at most $W$, with mean $\mu(\phi,a)$. Equivalently, one can associate
an independent reward sequence with each pair and reveal its next entry
when that pair is selected. Assume that previous returns are available
before the next selection. Define
\[
  \mu^*(\phi)=\max_{a\in\mathcal A}\mu(\phi,a),\qquad
  \Delta(\phi,a)=\mu^*(\phi)-\mu(\phi,a).
\]
Here the comparator maximizes a stationary per-decision return at a fixed
state. It is not the optimal policy of the original sequential memory
management problem, where actions can change future states and rewards.

Initialize each action once at a given state. At local decision $t$, let
$N_a(t-1)$ be its observed sample count and $\widehat\mu_a(t-1)$ the
arithmetic mean of those samples. The reference UCB1 rule is
\begin{equation}
  a_t\in\arg\max_{a\in\mathcal A}
  \left\{\widehat\mu_a(t-1)
     +W\sqrt{\frac{2\ln t}{N_a(t-1)}}\right\}.
  \label{eq:ucb-rule}
\end{equation}
Rescaling bounded returns to $[0,1]$ and applying the standard UCB1
argument~\citep{auer2002finite,lattimore2020bandit} gives, for a
deterministic local horizon $n\geq|\mathcal A|$,
\begin{equation}
  \mathbb E[\mathcal R_n(\phi)]
  \leq
  \sum_{a:\Delta(\phi,a)>0}
  \left\{
    \frac{8W^2\ln n}{\Delta(\phi,a)}
    +\left(1+\frac{\pi^2}{3}\right)\Delta(\phi,a)
  \right\},
  \label{eq:theory-ucb-bound}
\end{equation}
where $\mathcal R_n(\phi)=\sum_{t=1}^n\Delta(\phi,a_t)$.
Briefly, after approximately $8W^2\ln n/\Delta^2$ samples of a suboptimal
action, selecting it requires a deviation of an empirical mean beyond
its confidence radius. Hoeffding's inequality and a union bound make the
remaining expected number of such selections finite. The bound applies
to the sample-mean rule above with the specified exploration coefficient.
It does not automatically apply after replacing the sample mean by an
exponentially weighted estimate or changing the coefficient.

Summing statewise exploration costs motivates a small number of contexts
and actions. It does not determine an optimal discretization: coarsening
can merge states whose useful actions differ, creating approximation
error absent from the bound. Similarly, two action configurations that
collapse to the same backend behavior have the same mean, but need not
have zero gap. If that behavior is suboptimal, separate aliases can each
incur exploration cost. Sharing their statistics would require an explicit
identification of equivalent actions.

\paragraph{What the constant-step-size update estimates.}
Consider one fixed state--action pair in isolation, and index its updates
by $k$. Under the additional idealization that its observed returns
$G_1,G_2,\ldots$ are independent and identically distributed, with mean
$\mu$ and variance $\sigma^2\leq W^2/4$, the update in
Eq.~\eqref{eq:q-update} is
\begin{equation}
  Q_{k+1}=(1-\alpha)Q_k+\alpha G_{k+1},\qquad 0<\alpha<1.
  \label{eq:theory-ema}
\end{equation}
For a deterministic initialization $Q_0$, expanding this recursion yields
\begin{align}
  \mathbb E[Q_k]-\mu
    &=(1-\alpha)^k(Q_0-\mu), \label{eq:theory-prior-decay}\\
  \operatorname{Var}(Q_k)
    &=\frac{\alpha\sigma^2}{2-\alpha}
       \left[1-(1-\alpha)^{2k}\right].
       \label{eq:theory-ema-variance}
\end{align}
Consequently,
\begin{equation}
  \mathbb E[(Q_k-\mu)^2]
  \leq (1-\alpha)^{2k}(Q_0-\mu)^2
       +\frac{\alpha W^2}{4(2-\alpha)}.
  \label{eq:theory-ema-mse}
\end{equation}
The mean forgets its initialization geometrically, while nonzero return
variance leaves a nonzero estimation-variance floor. This is a tracking
estimate, not almost-sure convergence to a fixed value at constant
$\alpha$. The warm-start values in the implementation are initial
estimates; they are not Bayesian pseudo-counts. Their direct contribution
after $k$ updates is $(1-\alpha)^kQ_0$.

For comparison, an arithmetic sample mean corresponds to a local step
size $\alpha_k=1/k$ and converges under independent stationary sampling
when that pair is observed infinitely often. More general stochastic
approximation results require suitable noise conditions and diminishing
step sizes whose sum diverges and squared sum converges \emph{on each
pair's own update sequence}. Conditions on a global clock alone do not
ensure these properties on a sparsely visited subsequence. UCB can explore
actions when a state recurs; it cannot ensure that the environment revisits
every state indefinitely.

\paragraph{Limits for sequential memory control.}
Several aspects of deployment violate the reference assumptions. The
memory store and downstream policy evolve, so the return distribution
for the same coarse state--action key can change. Several decisions in
one episode share the same terminal reward, creating dependent feedback,
and that feedback is unavailable until the task ends. The compact key is
not established to be a sufficient Markov state. These facts prevent the
stationary UCB1 bound from certifying the implemented controller or
implying convergence within any fixed number of tasks.

Finally, Eq.~\eqref{eq:q-update} credits a maintenance action only with
the current task's reward. A benefit appearing in a later task is not
propagated back to that action. Including the resulting memory size or
plan availability in later states can help later decisions, but does not
repair this missing cross-task credit. Long-term maintenance guarantees
would require an explicit treatment of such delayed effects. The results
above therefore provide conditional reference calculations and clarify
the behavior of the estimator; empirical performance and long-term
memory quality must be assessed separately.

\clearpage
\section{Prompt Templates}
\label{app:prompts}

This appendix lists representative solver, augmentation, and adapted baseline-module templates used by \ours and its evaluation harness.

\subsection{Benchmark Solver System Prompts}

All memory methods (including \ours) share the identical solver system prompt per benchmark, so that differences in reported performance isolate the memory component.

\begin{tcolorbox}[colback=gray!5,colframe=black!40,title={ALFWorld system prompt},fonttitle=\small\bfseries]
\footnotesize\ttfamily\raggedright
You are now in a household environment called Alfworld, and your tasks include locating objects, heating or cooling items, and other similar activities.\par\smallskip
NOTE:\par

- You must strictly follow the syntactic structure of the steps (where 'a' and 'b' are variables):\par

\hspace*{2em}1. take a from b.\par

\hspace*{2em}2. go to a.\par

\hspace*{2em}3. open a.\par

\hspace*{2em}4. put a in/on b. (always write "in/on" together, never "in" or "on" alone.)\par

\hspace*{2em}5. clean a with b.\par

\hspace*{2em}6. heat a with b.\par

\hspace*{2em}7. cool a with b.\par

\hspace*{2em}8. use a.\par

\hspace*{2em}9. think: xxx\par

- You must check carefully whether your output command is consistent with the allowed commands above. Any output not among the listed commands is rejected.
\end{tcolorbox}

\begin{tcolorbox}[colback=gray!5,colframe=black!40,title={ScienceWorld system prompt},fonttitle=\small\bfseries]
\footnotesize\ttfamily\raggedright
You are a helpful assistant to do scientific experiments in a text-based environment. In the environment, there are several rooms: kitchen, foundry, workshop, bathroom, outside, living room, bedroom, greenhouse, art studio, hallway. You should explore the environment and find the items you need to complete the experiment. You can teleport to any room in one step.\par\smallskip
For each turn, choose "Thought" or "Action". If "Thought", output "Thought: ... \textbackslash n Action: ...". If "Action", output only "Action: ...". Only one Action per response.\par\smallskip
Available actions: open / close / activate / deactivate / connect / disconnect / use / look around / examine / look at / read / move / pick up / pour / mix / teleport to / focus on / wait / wait1.\par\smallskip
CRITICAL RULES: use EXACT object names as shown in observations; "focus on OBJ" is typically required at the end; read the task description literally.
\end{tcolorbox}

\begin{tcolorbox}[colback=gray!5,colframe=black!40,title={PDDL Blocksworld system prompt},fonttitle=\small\bfseries]
\footnotesize\ttfamily\raggedright
You are a robot with four actions: pickup, putdown, stack, and unstack. Blocks can be stacked on top of each other; the arm holds at most one block; the table holds the rest.\par\smallskip
Actions:\par

- think: xxx  (format: `think: ...`)\par

- pickup <block>: pick up a clear block from the table if the arm is empty.\par

- putdown <block>: put the held block on the table.\par

- stack <block> <block>: stack held top-block onto a clear bottom-block.\par

- unstack <block> <block>: pick the top-block off a bottom-block when the arm is empty.\par\smallskip
You must strictly follow these actions; no other actions are allowed.
\end{tcolorbox}

\noindent Analogous solver system prompts are used for \textit{barman}, \textit{gripper}, and \textit{tyreworld} (see \texttt{tasks/prompts/pddl\_prompt.py} in the released code).

\paragraph{TriviaQA / WebWalkerQA / GAIA.}
The QA harness requests concise final answers (``Answer: X'' for multiple choice, a single string for open-answer tasks). These output conventions do not by themselves specify the complete system prompts, tool prompts, or task preprocessing.

\subsection{MemCon-Specific Prompt Fragments}

The following augmentation fragments are assembled from stored trajectories and concatenated to retrieved context without a separate LLM call to generate them.

\begin{tcolorbox}[colback=blue!2,colframe=blue!30,title={Injected by \textsc{PlanInject} / \textsc{Retrieve}},fonttitle=\small\bfseries]
\footnotesize\ttfamily\raggedright
[Proven plan for '\{goal\_type\}' tasks]\par

\hspace*{2em}1. \{step\_1\_generalized\}\par

\hspace*{2em}2. \{step\_2\_generalized\}\par

\hspace*{2em}...\par

\hspace*{2em}k. \{step\_k\_generalized\}\par

\hspace*{2em}Adapt object/location names to your current task.
\end{tcolorbox}

Here each step is produced by generalizing a successful trajectory: numbered location and object instances are replaced with categorical placeholders (e.g., \texttt{"go to shelf 3"}\,$\to$\,\texttt{"go to [shelf]"}, \texttt{"take mug 1 from diningtable 2"}\,$\to$\,\texttt{"take [mug] from [diningtable]"}). The regular-expression rewriter covers the standard ALFWorld/ScienceWorld vocabulary (shelves, cabinets, drawers, containers, common foods, tools, etc.).

\begin{tcolorbox}[colback=blue!2,colframe=blue!30,title={Injected for ALFWorld \emph{puttwo} and analogous composite goals},fonttitle=\small\bfseries]
\footnotesize\ttfamily\raggedright
[TWO-OBJECT TASK: Complete all steps for object 1 first, then repeat for object 2]\par

\hspace*{2em}1. Find \& take first object $\to$ put it at target.\par

\hspace*{2em}2. Find \& take second object $\to$ put it at target.\par\smallskip
{[Reference: single-put plan]}\par

\hspace*{2em}1. \{step\_1\}\par

\hspace*{2em}2. \{step\_2\}\par

\hspace*{2em}...
\end{tcolorbox}

The second block reuses the generalized plan for the simpler single-object variant (\textit{put}) when one has been learned, providing an additional hint to the LLM agent.

\subsection{Baseline Memory-Module Prompts}

We summarize the memory-write/re-rank prompts used by the baselines we re-implemented on top of the shared two-method memory interface. Full text of these prompts is in the released code. These templates document adapted modules; they do not establish reproduction of every learned component or prompt in the cited systems.

\begin{tcolorbox}[colback=gray!3,colframe=black!30,title={Trajectory-summarization system prompt},fonttitle=\small\bfseries]
\footnotesize\ttfamily\raggedright
You are a helpful assistant that writes a description of the task resolution trajectory.\par

1) Try to summarize the trajectory in no more than 6 sentences.\par

2) Your response should be a single line of text.
\end{tcolorbox}

\begin{tcolorbox}[colback=gray!3,colframe=black!30,title={Relevance-scoring system/user prompts},fonttitle=\small\bfseries]
\footnotesize\ttfamily\raggedright
System: You are an agent designed to score the relevance between two pieces of text.\par\smallskip
User: You will be given a successful case and an ongoing task. Do not summarize either case; evaluate how relevant and helpful the successful case is for the ongoing task, on a scale of 1--10.\par

Success Case: \{trajectory\}\par

Ongoing task: \{query\}\par

Score:
\end{tcolorbox}

\begin{tcolorbox}[colback=gray!3,colframe=black!30,title={Phase-based summarization system prompt},fonttitle=\small\bfseries]
\footnotesize\ttfamily\raggedright
You are an agent skilled in summarization. Your task is to generate phase-based summaries from given execution records of an agent's task. These summaries help the agent efficiently utilize existing information, avoid redundant computations, and ensure task continuity.\par\smallskip
1. Phase-based summarization: organize records into logical phases and extract key steps.\par

2. Task relevance: explain what has been completed and what remains.\par

3. Clarity and conciseness: precise language, no unnecessary details.\par\smallskip
If intermediate states are incorrect or irrelevant, filter or correct them to make the summary more accurate.
\end{tcolorbox}

\begin{tcolorbox}[colback=gray!3,colframe=black!30,title={Rule-comparison prompt for the adapted OAgent module},fonttitle=\small\bfseries]
\footnotesize\ttfamily\raggedright
You are an advanced reasoning agent that derives general rules from examples. You will receive one successful trial and one failed trial.\par\smallskip
Goal: compare the positive and negative examples to extract insights. The insights must be concise and expressed as high-level reasoning principles.\par\smallskip
Output format: each line is one of\par

\hspace*{2em}AGREE <EXISTING RULE NUMBER>: <EXISTING RULE>\par

\hspace*{2em}REMOVE <EXISTING RULE NUMBER>: <EXISTING RULE>\par

\hspace*{2em}EDIT   <EXISTING RULE NUMBER>: <NEW MODIFIED RULE>\par

\hspace*{2em}ADD: <NEW RULE>
\end{tcolorbox}

\end{document}